%% file: main.tex
\documentclass[11pt]{article}
\usepackage[utf8]{inputenc}
\input{preamble.sty}

\title{Margin-Independent Online Multiclass Learning via Convex Geometry}

\author{Guru Guruganesh  \\ Google Research \and Allen Liu \\ MIT \and Jon Schneider \\ Google Research \and Joshua Wang \\ Google Research}
\date{}

\begin{document}

\maketitle

\begin{abstract}
    \input{abstract}
\end{abstract}

\input{intro}
\input{model}
\input{euclidean-norm}
\input{general-norm}
\input{general-regions}

\bibliographystyle{alpha}
\bibliography{cluster.bib}

\newpage 
\appendix

\input{contextual_search}
\input{apx-euclidean-norm}

\input{apx-mistake-bounds}
\input{apx-general-norm}
\input{apx-general-regions}
\input{apx-regret-vs-distance}

\end{document}

%% file: abstract.tex
We consider the problem of multi-class classification, where a stream of adversarially chosen queries arrive and must be assigned a label online. Unlike traditional bounds which seek to minimize the misclassification rate, we minimize the total distance from each query to the region corresponding to its correct label. When the true labels are determined via a nearest neighbor partition -- i.e. the label of a point is given by which of $k$ centers it is closest to in Euclidean distance -- we show that one can achieve a loss that is independent of the total number of queries. We complement this result by showing that learning general convex sets requires an almost linear loss per query. Our results build off of regret guarantees for the geometric problem of contextual search. In addition, we develop a novel reduction technique from multiclass classification to binary classification which may be of independent interest. 

%% file: intro.tex
\section{Introduction}
Online multiclass classification is a ubiquitous problem in machine learning. In this problem, a learning algorithm is presented with a stream of incoming query points and is tasked with assigning each query with a label from a fixed set. After choosing a label, the algorithm is told the true label of the query point. The goal of the algorithm is to learn over time how to label the query points as accurately as possible. 

Traditionally, theoretical treatments of this problem are built around the notion of a \textit{margin} $\gamma$. This margin represents the extent to which the input points are well-separated from the boundaries between different labels. For example, the analysis of the classic Perceptron algorithm \citep{novikoff1963convergence} guarantees that it makes at most $O(1/\gamma^2)$ mistakes when performing binary classification, as long as all query points are distance at least $\gamma$ from a hyperplane separating the two classes. More sophisticated analyses and algorithms (relying on e.g. hinge loss) do not necessarily assume the classes are as well separated, but still inherently incorporate a margin $\gamma$ (for example, the hinge loss associated with a point is positive unless it is $\gamma$-separated).

In this paper, we present an alternative to the traditional margin approaches. Our approach weights each mistake by how ambiguous the classification task is for that point, rather than penalizing all mistakes equally. More precisely, consider a partition of the space of all possible query points into $k$ regions $R_i$, where $R_i$ contains all query points whose true label is $i$. In our formulation assigning a query point $q$ a label $i$ incurs a loss of $\ell(q, R_i)$, where $\ell(q, R_i)$ should be thought of as \textit{the distance needed to move $q$ so that it lies in $R_i$} (i.e., for it to be labelled correctly). For example, in the case of a linear classifier, $\ell(q, R_i)$ is zero if $q$ is correctly classified, and the distance to the classifier if $q$ is incorrectly classified. The goal of the algorithm is to minimize the total loss.

This notion of loss not only measures the rate of errors but also the degree of each error; choosing a wildly inaccurate label is punished more than selecting a label that is "almost" correct. This fine-grained approach to looking at errors has occurred in other areas of machine learning research as well. For example, the technique of knowledge distillation is based on training a smaller model on the logits produced by a larger model~\citep{hinton2015distilling}. Hinton et al. explain, ``The relative probabilities of incorrect answers tell us a lot about how the cumbersome model tends to generalize. An image of a BMW, for example, may only have a very small chance of being mistaken for a garbage truck, but that mistake is still many times more probable than mistaking it for a carrot.'' Rather than leaning on the power of a trained model, our framework differentiates between these different incorrect answers based on the geometry of the problem.

\jonnote{try to justify our loss some more?}

\subsection{Our Results}

\subsubsection{Learning Linear Classifiers}

In this case we have a binary classification problem, where the two regions $R_1$ and $R_2$ are separated by an unknown $d$-dimensional hyperplane $\langle v, x\rangle = 0$ (with $||v||_2 = 1$). Our loss function in this case is the function $\ell(q, R_i) = |\langle q, v\rangle| \cdot \mathbf{1}(q \not\in R_i)$. We prove the following result:

\begin{theorem}[Restatement of Corollary \ref{cor:two-point}]\label{thm:intro1}
  There exists an efficient algorithm for learning a linear classifier that incurs a total loss of at most $O(d \log d)$. 
\end{theorem}

Note that the total loss in Theorem \ref{thm:intro1} is \textit{independent} of the time horizon (number of rounds) $T$. More importantly, note that this is stronger than the naive guarantee implied by the margin bounds for the Perceptron algorithm. Indeed, each mistake at a distance $\gamma$ from the separating hyperplane is assigned loss $O(\gamma)$ in our model. Since the Perceptron algorithm can make up to $O(1/\gamma^2)$ such mistakes, this only implies Perceptron incurs a total loss of at most $O(1/\gamma)$ (which blows up as $\gamma \rightarrow 0$). 

Indeed, our algorithm in Theorem \ref{thm:intro1} is not based off of the Perceptron algorithm or its relatives, but rather off of recently developed algorithms for a problem in online learning known as \textit{contextual search}. In contextual search, a learner is similarly faced with a stream of incoming query points $\query{t}$ and wishes to learn a hidden vector $v$. However, instead of trying to predict the sign of $\inner{v}{q_t}$, in contextual search the goal is to guess the \textit{value} of $\dot{v}{q_t}$. After the learner submits a guess, they are told whether or not their guess was higher or lower than the true value of $\dot{v}{q_t}$ (and pay a loss equal to the distance between their guess and the truth). The best known contextual search algorithms rely on techniques from integral geometry (bounding various intrinsic volumes of the allowable knowledge set), and are inherently different than existing Perceptron/SVM-style algorithms.

While it may seem like contextual search (which must predict the value of $\dot{v}{q_t}$ instead of just the sign) is strictly harder than our online binary classification problem, they are somewhat incomparable (for example, unlike in contextual search, we have no control over what feedback we learn about the hidden vector $v$). Nonetheless, in Theorem \ref{thm:csearch_reduction} we show a general reduction from our binary classification problem to contextual search. This allows us to use recent results of \cite{liu2020optimal} to obtain our $O(d\log d)$ bound in Theorem \ref{thm:intro1}.
\subsubsection{Learning Nearest Neighbor Partitions}

One natural way to split a query space into multiple classes is via a \textit{nearest neighbor partition}. In this setting, each label class $i$ is associated with a ``center'' $\prototype{i} \in \R^d$, and each region $R_i$ consists of the points which are ``nearest'' to $\prototype{i}$. To define ``nearest'', we introduce a similarity metric $\Loss(x, y)$ representing the ``distance'' between points $x$ and $y$ in $\R^d$. The two major classes of similarity metrics we consider are: a) the \textit{inner-product similarity} $\Loss(x, y) = -\langle x, y\rangle$ and b) the \textit{$L^p$ similarity} $\Loss(x, y) = ||x - y||_p$. Given a fixed similarity metric $\Loss$, our loss function in this case is the function $\ell(q, R_i) = \Loss(q, x_i) - \min_{i^*}\Loss(q, x_{i^*})$; in other words, the difference between the similarity between $q$ and $x_i$ with the similarity between $q$ and its most similar center\footnote{Technically, this is not quite the ``distance'' from $q$ to the correct region. We discuss in Appendix \ref{app:regret-vs-distance} why we choose this definition of loss for the nearest neighbor setting.}.

\begin{theorem}[Restatement of Corollary \ref{cor:k-point}]\label{thm:intro2}
\sloppy{For inner-product similarity, there exists an efficient randomized algorithm for learning a nearest neighbors partition that incurs a total expected loss of at most $O(k^2d\log d)$.}
\end{theorem}

Like some other algorithms for multiclass classification, our algorithm in Theorem \ref{thm:intro2} works by running one instance of our binary classification algorithm (Theorem \ref{thm:intro1}) for each pair of labels. Unlike some other ``all-vs-all'' methods in the multiclass classification literature, however, it does not suffice to run a simple majority vote over these instances. Instead, to prove Theorem \ref{thm:intro2}, we solve a linear program to construct a probability distribution over centers that guarantees that our expected loss is bounded by an expected decrease in a total potential of all our $\binom{k}{2}$ sub-algorithms. 

Our results for $L^p$ similarity are as follows:

\begin{theorem}[Restatement of Theorems \ref{thm:special-pnorm-kpoint} and \ref{thm:general-p-norm-k-point}]\label{thm:intro3}
For $L^p$ similarity, when $p$ is a positive even integer, there exists an efficient randomized algorithm for learning a nearest neighbors partition that incurs a total expected loss of at most $O(k^2\poly(p, d))$. 

For an arbitrary $p \geq 2$, if all $k$ centers are $\Delta$-separated in $L^p$ distance, there exists an efficient randomized algorithm that incurs a total expected loss of
$$\frac{k^2 \poly(p, d) }{\Delta} \cdot \left(\frac{1}{p-2}\right)^2 \,.$$
\end{theorem}

When $p$ is an even integer, it is possible to construct a polynomial kernel that exactly reduces this problem to the problem for inner-product similarity (albeit in the higher dimensional space $\R^{d(p+1)}$). When $p$ is not an even integer, it is no longer possible to perform an exact reduction to the inner-product similarity problem. Instead, in parallel we perform a series of approximate reductions to inner-product similarity at multiple different scales (the full algorithm can be found in Appendix \ref{app:pnorms}). Surprisingly, this technique only gives $T$-independent bounds on the loss when $p \geq 2$. It is an interesting open problem to develop algorithms for the case $1 \leq p < 2$ (and more generally, for arbitrary norms). 

\subsubsection{Learning General Convex Regions}

Finally, we consider the case where the regions $R_i$ are not defined in relation to hidden centers, but where they can be any convex subsets of $\R^d$. Interestingly, in this case it is impossible to achieve total loss independent of $T$. Indeed, we prove a lower bound of $\Omega(T^{1-O(1/d)})$ for the total loss of any algorithm for this setting, even when $k=2$. 

\begin{restatable}{theorem}{generalregionstheorem}
\label{thm:general-regions}
  Any algorithm for learning general convex regions incurs a total loss of at least \\ $\Omega\left(T^{(d-4) / (d-2)}\right)$, even when there are only two regions.
\end{restatable}

\subsection{Mistake Bounds and Halving Algorithms}

As mentioned in the introduction, classical algorithms for multi-class classification generally try to minimize the \textit{mistake bound} (the total number of classification errors the algorithm makes) under some margin guarantee $\gamma$. Even though our algorithms are designed to minimize the absolute loss and not a margin-dependent mistake bound, it is natural to ask whether our algorithms come with any natural mistake bound guarantees.

We show that our algorithms do in fact possess strong mistake bounds, matching the dependence on the margin $\gamma$ of the best known halving algorithms. In particular, we show the following.

\begin{theorem}[Restatement of Theorem \ref{thm:mistake1}]\label{thm:intro3a}
  If all query points $q_t$ are at least distance $\gamma$ away from the separating hyperplane, our algorithm for learning linear classifiers (Theorem \ref{thm:intro1}) makes at most $O(d\log (d/\gamma))$ mistakes.
\end{theorem}

\begin{theorem}[Restatement of Theorem \ref{thm:mistake2}]\label{thm:intro4}
  If all query points $q_t$ are at least distance $\gamma$ away from the boundary between any two regions, our algorithm for learning nearest neighbor partitions (Theorem \ref{thm:intro2}) makes at most $O(k^2d\log(d/\gamma))$ mistakes.
\end{theorem}


In comparison, the best dimension-dependent classical bounds for this problem come from halving algorithms (efficiently implementable via linear programming) which have mistake bounds of $O(d\log (1/\gamma))$ and $O(kd\log(1/\gamma))$ respectively. In the first case, our mistake bound is nearly tight (losing only an additive $O(d\log d)$). In the second case, our mistake bound is tight up to a multiplicative factor of $k$; it is an interesting open question whether it is possible to remove this factor of $k$ in our techniques.

We additionally introduce a variant on the mistake bound that we call a \textit{robust mistake bound} that is defined as follows. Normally, when we have a margin constraint of $\gamma$, we insist that all query points $q_t$ are distance at least $\gamma$ away from any separating hyperplane. In our robust model, we remove this constraint, but only count mistakes when the query point $q_t$ lies at least $\gamma$ away from the separating hyperplane.

Existing algorithms (both the Perceptron and halving algorithms) do not appear to give any non-trivial mistake bounds in the robust model -- it is very important for the analysis of these algorithms that \textit{every} query point is far from the separating hyperplane. On the other hand, it straightforwardly follows from our notion of loss that if we have an $O(R)$-loss algorithm for some problem, that algorithm simultaneously achieves an $O(R/\gamma)$ robust mistake bound. In particular, we obtain robust mistake bounds of $O((d\log d)/\gamma)$ and $O((k^2d\log d)/\gamma)$ for learning linear classifiers and learning nearest neighbor partitions respectively. 

\paragraph{Related work}
The problem of online binary classification (and specifically, of online learning of a linear classifier) is one of the oldest problems in machine learning. The Perceptron algorithm was invented in \cite{rosenblatt1958perceptron}, and the first mistake bound analysis of the Perceptron algorithm appeared in \cite{novikoff1963convergence}. Since then, there has been a myriad of research on this problem, some of which is well-surveyed in \cite{mohri2013perceptron}. Of note, the first mistake bounds for the non-separable case appear in \cite{freund1999large}. We are not aware of any work on this problem that investigates the same loss we present in this paper. 

Bounds for support vector machines (see~\cite{vapnik2013nature}) also result in the use of a margin to bound the number of mistakes. Choosing a suitable kernel can help create or improve the margin when viewing the data in the ``kernel'' space. We use a similar technique in proving bounds for generalized $\lp$ norms. Our idea is to use differently-scaled kernels to produce increasingly accurate approximations. To the best of our knowledge, the technique we present is novel. There are other related techniques in the literature, e.g.~\cite{srebro2006learning} attempt to learn the best kernel to improve classification error.

Similarly, there is a wealth of both theoretical and empirical research on multiclass classification. As far back as 1973, researchers were looking at generalizing binary classification algorithms to this setting (see e.g. Kesler's construction in \cite{duda1973pattern}). One piece of relevant work is \cite{crammer2003ultraconservative}, which generalizes online binary classification Perceptron-style algorithms to solve online multiclass classification problems -- in particular, the multiclass hypotheses they consider are same as the nearest neighbor partitions generated by the inner-product similarity metric (although as with the Perceptron, they only analyze the raw classification error). Since then, this work has been extended in many ways to a variety of settings (e.g. \cite{crammer2006online, crammer2013multiclass, kakade2008efficient}).



Another way of looking at the problem of multiclass classification is that we are learning how to cluster a sequence of input points into $k$ pre-existing clusters. Indeed, the nearest neighbor partition with $L^2$ similarity metric gives rise to exactly the same partitioning as a $k$-means clustering. There is an active area of research on learning how to cluster in an online fashion (e.g. \cite{gentile2014online, gentile2017context, bhaskara2020robust}). Perhaps most relevantly, \cite{liberty2016algorithm} studies a setting where one must choose cluster labels for incoming points in an online fashion, and then at the end of the algorithm can choose the $k$ centers (the goal being to minimize the total $k$-means loss of this eventual clustering). 


The algorithms we develop in this paper are based off of algorithms for contextual search. Contextual search is a problem that originally arose in the online pricing community \citep{cohen2016feature, lobel2016multidimensional}. The variant of contextual search we reference in this paper (with symmetric, absolute value loss) first appeared in \cite{Intrinsic18}. The algorithms in this paper were later improved in \cite{liu2020optimal} (and it is these improved algorithms that we build off of in this paper). 

%% file: model.tex
\section{Model}

\paragraph{Notation} We will write $\Ball_d(c, r)$ to denote the $d$-dimensional ball in $\R^d$ centered at $c$ with radius $r$. We write $\Ball_d$ in place of $\Ball_d(0, 1)$ to indicate the unit ball centered at the origin.

Given an $x \in \R^d$, we will write $\pnorm{x}{p}$ to denote the $\lp$ norm of the vector $x$. In the case of $p = 2$, we will often omit the subscript and write $\norm{x}$ in place of $\pnorm{x}{2}$.

\subsection{Online Multiclass Learning}

We will view the problem of online multiclass learning as follows. There are $k$ disjoint regions in some domain, say $\Domain$, labelled $\region{1}$ through $\region{k}$. The region $\region{i}$ contains the points in $\Ball_d$ that should be assigned the label $i$. The goal of the learner is to learn these subsets (and thus how to label points in $\Domain$ in an online manner). Every round $t$, the learner receives an adversarially chosen query point $\query{t} \in \Domain$. The learner must submit a prediction $I_t \in [k]$ for which region $\region{I_t}$ the point $\query{t}$ lies in. The learner then learns which region $\region{I^*_t}$ the point actually belongs to, and suffers some loss $\ell(\query{t}, \region{I_t})$. This loss function should in some way represent how far $\query{t}$ was from lying in the region $\region{I_t}$ chosen by the learner; for example, in the case where the learner chooses the correct region $\region{I^*_t}$, $\ell(\query{t}, \region{I^*_t})$ should be zero.

In this paper, we will consider two specific cases of the above learning problem. In the first case, we wish to learn a \textit{nearest-neighbor partition}. That is, the $k$ regions are defined by $k$ ``centers'' $\prototype{1}, \prototype{2}, \dots, \prototype{k} \in \Domain$. Region $\region{i}$ then consists of all the points which are ``nearest'' to center $\prototype{i}$ according to some similarity metric $\Loss(x, y)$ (where lower values of $\Loss(x, y)$ mean that $x$ and $y$ are more similar; note that $\Loss(x, y)$ \textit{does not} need to be an actual metric obeying the triangle-inequality). Given a similarity metric $\Loss(x, y)$, the loss our algorithm incurs when labelling query $\query{}$ with label $i$ is given by $\ell(\query{}, \region{i}) = \Loss(\query{}, \prototype{i}) - \Loss(\query{}, \prototype{i^*})$, where $\region{i^*}$ is the region containing query $\query{}$. 

We will examine several different possibilities for $\Loss(x, y)$, including:

\begin{itemize}
\item $\Loss(x, y) = - \langle x, y \rangle$. We refer to this as the \textit{inner-product similarity} between $x$ and $y$. Note that when $k = 2$, using this similarity metric reduces to the problem of learning a linear classifier. For $k > 2$, this results in similar partitions to those learned by multiclass perceptrons / SVMs \citep{crammer2003ultraconservative, crammer2006online}.
\item $\Loss(x, y) = ||x - y||_2$; in other words, the Euclidean distance between $x$ and $y$. When using this loss function, the $k$ regions are given by the Voronoi diagram formed by the $k$ centers $\prototype{i}$. 
\item For $p \geq 1$, $\Loss(x, y) = ||x - y||_p$; in other words, the $\lp$ distance between $x$ and $y$.
\end{itemize}

There is a straightforward reduction from the Euclidean distance similarity to the inner-product similarity (see Appendix \ref{app:euclidean}), so in Section \ref{sec:inner-product} we will primarily concern ourselves with the inner-product similarity. In Section \ref{sec:pnorms} we will tackle this problem for the case of general $\lp$ norms; for some cases (even integer $p$) it is possible to perform a similar reduction to the inner product similarity, but in general it is not and we will need to rely on other techniques.

In the second case, we wish to learn \textit{general convex sets}. In particular, there are $k$ disjoint convex sets $\region{1}, \dots, \region{k} \subseteq \Domain$ (not necessarily a partition). Each round $t$, we receive a query point $\query{t} \in \bigcup_{i=1}^k \region{i}$, guess a region $i \in [k]$, and are penalized against the loss function $\ell(\query{t}, \region{i}) \triangleq \min_{x \in \region{i}} \pnorm{\query{t} - x}{2}$. In other words, we are penalized the minimum distance from $\query{t}$ to the predicted region $\region{i}$, which is zero if our guess was correct. In Section \ref{sec:general-regions} we will show that (even in the case of $k=2$), there is no low-loss learning algorithm for learning general convex sets (in contrast to learning nearest neighbors).

Finally, for any $\alpha > 0$ and existing loss function $\ell(g, R_i)$, we can consider the modified loss function $\ell'(q, R_i) = \ell(q, R_i)^{\alpha}$. Note that this does not change the division of $\Domain$ into regions, but it \textit{can} change the total loss incurred by our algorithm (and will be useful for some of our reductions). If in any result we do not specify an $\alpha$, that means we are taking $\alpha = 1$ (i.e. the unmodified loss function). 






\subsection{Contextual Search}\label{sec:model_csearch}

One of the main tools we will rely on is an existing algorithm for a problem in online learning known as \textit{contextual search}. For our purposes, the problem of contextual search can be defined as follows. There is a hidden point $p \in \Ball_d$, unknown to the learner. Every round $t$ (for $T$ rounds) an adversary provides the learner with a query vector $x_t \in \Ball_d$. The learner must then submit a guess $g_t$ for the value of the inner product $\langle x_t, p \rangle$. The learner then learns whether their guess was too high or too low. At the end of the game, the learner incurs loss $\ell(g_t, \langle x_t, p\rangle) = |g_t - \langle x_t, p \rangle|$ for each of their guesses. The learner's goal is to minimize their total loss.

Interestingly, there exist algorithms for contextual search with total loss polynomial in the ambient dimension $d$ and independent of the time horizon $T$. The first such algorithms were given in \cite{Intrinsic18} and used techniques from integral geometry to obtain a total regret of $O(d^4)$. More recently, these algorithms were improved in \cite{liu2020optimal} to achieve a regret bound of $O(d\log d)$. We will rely on a slightly strengthened variant of the result from \cite{liu2020optimal} to work when the loss function raised to an arbitrary power.

\begin{theorem}\label{thm:contextual_search}
Let $\alpha > 0$. There exists an algorithm for contextual search with loss function $\ell(g_t, \langle x_t, p\rangle) = |g_t - \langle x_t, p \rangle|^{\alpha}$ that incurs a total loss of at most $O(\alpha^{-2}d\log d)$.
\end{theorem}

In particular, Theorem \ref{thm:contextual_search} is satisfied by the algorithm from \cite{liu2020optimal}. For completeness, we include a description of the algorithm (along with the proof of Theorem \ref{thm:contextual_search}) in Appendix \ref{app:contextual_search}. 






%% file: euclidean-norm.tex
\section{Learning Nearest Neighbor Partitions}\label{sec:inner-product}

\subsection{The Two-Point Case: Learning a Hyperplane}\label{sec:two-point}

To begin, we will discuss how to solve the $k = 2$ variant of the problem of learning nearest neighbor partitions for the inner-product similarity function $\Loss(q, x) = -\inner{q}{x}$. Recall that in this setting we have two unknown centers $x_1$ and $x_2$ belonging to $\Ball_{d}$. Each round $t$ we are given a query point $q_t \in \Ball_{d}$, and asked to choose a label $I_t \in \{1, 2\}$ of the center that we think is most similar to $q_t$ (i.e., that maximizes $\Loss(q_t, x)$). Letting $y_t = x_{I_t}$ and $x^*_t = \arg\max_{x_i} \Loss(q, x_i)$, our loss in round $t$ is zero if we guess correctly ($y_t = x^*_t$) and is $\Loss(q, x^*_t) - \Loss(q, y_t)$ if we guess incorrectly (we will also be able to deal with the case when the loss is $\abs{\Loss(q, x^*_t) - \Loss(q, y_t)}^{\alpha}$ for some $\alpha > 0$). Afterwards, we are told the identity (but not the location) of $x^*_t$. 

Note that the optimal strategy in this game (for an agent who knows the hidden centers $x_1$ and $x_2$) is to guess $I_t = 1$ whenever $\langle q, x_1 \rangle > \langle q, x_2 \rangle$, and to guess $I_t = 2$ otherwise. Rewriting this, we want to guess $I_t = 1$ exactly when $\langle q, x_1 - x_2 \rangle > 0$. If we let $w = x_1 - x_2$, we can think of goal as learning the hyperplane $\langle q, w \rangle = 0$. More specifically, each round we are given a point $q$ and asked which side of the hyperplane $q$ lies on. If we guess correctly, we suffer no loss; if we guess incorrectly, we suffer loss equal to the distance from $q$ to this hyperplane. In either case, we learn afterwards which side of the hyperplane $q$ lies on.

\subsubsection{Reducing to Contextual Search}

We will show that we can solve this online learning problem with total loss $O(\poly(d))$, independent of the number of rounds $T$ in the time-horizon. As mentioned earlier, our primary tool will be existing algorithms for a problem in online learning known as \textit{contextual search}.

Recall that in contextual search there is a hidden vector $v \in \Ball_{d}$, and each round we are given a vector $q_t \in \Ball_{d}$. However, unlike in our problem (where we only care about the sign of the inner product $\langle q_t, v\rangle$), the goal in contextual search is to submit a guess $g_t$ for the value of the inner product $\langle q_t, v\rangle$. We then incur loss equal to the absolute distance $|\langle q_t, v_t \rangle - g_t|$ between our guess and the truth, and are then told whether our guess $g_t$ was too high or too low. 

As mentioned earlier in Section \ref{sec:model_csearch}, there exist algorithms for contextual search with $O(d\log d)$ total loss. Via a simple reduction, we will show that we can apply these algorithms in our setting.

\begin{theorem}\label{thm:csearch_reduction}
Fix $\alpha > 0$. Assume there exists an algorithm $\mathcal{A}$ for contextual search with loss function $\ell(g_t, \dot{x_t}{p}) = \left|g_t - \dot{x_t}{p}\right|^{\alpha}$ that incurs regret at most $R(d, T)$. Then there exists an algorithm $\mathcal{A}'$ that incurs regret at most $R(d, T)$ for the $k=2$ case of learning nearest neighbor partitions with similarity metric $\Loss(x, y) = -\langle x, y\rangle$ and loss raised to the power $\alpha$.
\end{theorem}
\begin{proof}
The general idea behind this reduction is simple. We will run our algorithm $\mathcal{A}$ for contextual search to find the hidden point $w = x_1 - x_2$. Whenever we are given a query point $q_t$ and need to guess the sign of $\langle q_t, w\rangle$, we will ask our contextual search algorithm $\mathcal{A}$ for a guess $g_t$ for the value of $\langle q_t, w \rangle$. If $g_t > 0$, we will guess that $\langle q_t, w \rangle > 0$; otherwise, we will guess that $g_t < 0$. 

There is one important caveat here: how do we update the algorithm $\mathcal{A}$? Recall that any contextual search algorithm expects binary feedback each round as to whether its guess $g_t$ was too high or too low. While in some cases, we can provide $\mathcal{A}$ with accurate feedback, in many cases we cannot: for example, if $\mathcal{A}$ submits a guess $g_t = 2.5$ for $\langle q_t, w\rangle$, and all we learn is that $\langle q_t, w\rangle \geq 0$, we cannot say with confidence whether $\mathcal{A}$'s guess was too large or not.

The solution to this is to \textit{only update the state of $\mathcal{A}$ on rounds where we guess the sign incorrectly}. Note that for such rounds, we definitively know whether $g_t$ was too high or too low; for example, if $\mathcal{A}$ guessed $g_t = 2.5$ and hence we guess $\langle q_t, w \rangle > 0$, but it turns out that $\langle q_t, w \rangle < 0$, we know for certain that the guess $g_t$ was too high. On all other rounds we do not update the state of $\mathcal{A}$, effectively rolling back the state of $\mathcal{A}$ to before we asked the question about $q_t$. This means that the effective number of rounds $\mathcal{A}$ experiences (gets feedback on) may be less than $T$; nonetheless, since $R(d, T)$ is non-decreasing in $T$, the total loss of $\mathcal{A}$ on these rounds is still at most $R(d, T)$.

Finally, we will relate the loss of our algorithm $\mathcal{A'}$ for learning nearest neighbors to the loss of the contextual search algorithm $\mathcal{A}$. To start, note that we only sustain loss in rounds when we guess the sign of $\langle q_t, w\rangle$ incorrectly. Luckily, these rounds happen to be exactly the rounds where we update the state of $\mathcal{A}$ (and thus the rounds whose loss counts towards the $R(d, T)$ bound). In a round where we guess the sign incorrectly, $\mathcal{A'}$ sustains a loss of $|\langle q_t, w\rangle|^{\alpha}$, and $\mathcal{A}$ sustains a loss of $|\langle q_t, w \rangle - g_t|^{\alpha}$. Since $\sign(g_t) \neq \sign(\langle q_t, w \rangle)$, this means that $|\langle q_t, w \rangle - g_t| \geq |\langle q_t, w\rangle|$, and therefore $\mathcal{A}$ sustains more loss than $\mathcal{A'}$. It follows that the total loss sustained by $\mathcal{A'}$ is at most the total loss sustained by $\mathcal{A}$ on this set of rounds, which in turn is at most $R(d, T)$. 
\end{proof}

\begin{corollary}\label{cor:two-point}
Fix an $\alpha > 0$. When $k=2$, there exists an algorithm for learning nearest neighbor partitions with similarity metric $\Loss(x, y) = -\langle x, y\rangle$ and loss raised to the power $\alpha$ that incurs total loss at most $O(\alpha^{-2}d\log d)$.
\end{corollary}


\subsubsection{Potential-Based Algorithms}

In order to generalize this to $k > 2$ labels, we will need to open the black box that is our contextual search algorithm slightly. In particular, the argument in the following section requires our algorithm for the $k=2$ case to be a \textit{potential-based algorithm}. 

Before defining exactly what a potential-based algorithm is, it will be useful to define the notion of a \textit{knowledge set}. For the problem we are considering in this section -- the $k=2$ variant of our problem for inner-product similarity -- we will define the \textit{knowledge set} $K_t$ at time $t$ to be the set of possible values for $w = x_1 - x_2$ that are consistent with all known information thus far. Note that since $x_1$ and $x_2$ start as arbitrary points in $\Ball_{d}$, the knowledge set $K_0$ is simply the set $\Ball_{d} - \Ball_{d} = 2\Ball_{d}$. As the algorithm gets more feedback about $w$, the knowledge set shrinks; however, since this feedback is always of the form of a linear constraint (e.g. $\langle q_t, w \rangle \geq 0$), the knowledge set $K_t$ is always a convex subset of $\R^d$.

Let $S_t$ be the history of all feedback the algorithm has seen up to (but not including) round $t$; that is $S_t = \{(q_1, I^*_1), (q_2, I^*_2), \dots, (q_{t-1}, I^*_{t-1})\}$. Let $\mathcal{S}_t$ denote the set of possible values for $S_t$, and let $\mathcal{S} = \bigcup_{t} \mathcal{S}_t$. We can think of $S_t$ as capturing the \textit{state} of a deterministic algorithm at time $t$. For now, it is fine to think of $S_t$ as interchangeable with $K_t$; i.e., the knowledge set $K_t$ captures all relevant details about all feedback the algorithm has observed thus far. (Later, when looking at $\lp$ similarity metrics, we will want to keep track of separate knowledge sets at different scales, and thus will want a more nuanced notion of potential-based algorithm). 

\begin{definition} \label{defn:potential-algorithm}
A deterministic algorithm $\A$ (for $k=2$ and $\alpha > 0$) is a potential-based algorithm if there exists a potential function $\Potential$ from $\mathcal{S}$ to $\R_{\geq 0}$ and a ``loss bound'' function $L$ from $\mathcal{S} \times \Domain$ to $\R_{\geq 0}$ that satisfy:

\begin{itemize}
    \item For all rounds $t$, $\Potential(S_{t+1}) \leq \Potential(S_t)$.
    \item Let $q_t$ be the query point in round $t$. Then $L(S_t, q_t)$ is an upper bound on the loss incurred by \textbf{any guess}. In other words, $L(S_t, q_t)$ must satisfy
    
    \[
    L(S_t, q_t) \geq \max_{\substack{x_1, x_2  \\ \text{consistent with } S_t}} \left \lvert  \delta(q_t, x_1) - \delta(q_t, x_2) \right \rvert^{\alpha} \,.
    \]
    
    \item Again, let $q_t$ be the query point in round $t$.  If $\A$ guesses the label incorrectly in round $t$, then $\Potential(S_{t}) - \Potential(S_{t+1}) \geq L(S_t, q_t)$.

\end{itemize}
\end{definition}

For the case of inner-product similarity, we will set the loss bound function $L(S_t, q_t)$ equal to the width of the knowledge set $K_t$ in the direction $q_t$. Importantly, this choice of $L$ is an efficiently computable (in terms of $q_t$ and the knowledge set $K_t$) upper-bound on the loss, which will prove important in the following section (in general, we will want both $\Phi$ and $L$ to be efficiently computable in order to efficiently carry out the reduction in Section \ref{sec:multi-reduction}).

Note also that such a potential immediately gives a way to bound the total loss of an algorithm independently of $T$; in particular, summing the inequality $\Potential(S_t) - \Potential(S_{t+1}) \geq L(S_t, q_t)$ over all $t$ gives that the total loss is at most $\Potential(S_0)$. We call the value $\Potential(S_0)$ the \textit{initial potential} of the algorithm $\A$. 

Similar potential-based arguments are used in \cite{liu2020optimal} and \cite{Intrinsic18} to give $T$-independent total loss bounds for the problem of contextual search. Unsurprisingly, these arguments can be extended (via Theorem \ref{thm:csearch_reduction}) to apply to the problem of learning nearest neighbors as well.

\begin{theorem} \label{thm:two-point}
Fix an $\alpha > 0$. When $k=2$, there exists a \emph{potential-based} algorithm for learning nearest neighbor partitions under the similarity metric $\Loss(x, y) = -\inner{x}{y}$ that incurs total loss at most $O(\alpha^{-2}d\log d)$ (independent of the time horizon $T$).
\end{theorem}

The proof of Theorem \ref{thm:two-point} can be found in Appendix \ref{app:potential-based}.

\subsection{From Two to Many Centers}
\label{sec:multi-reduction}

We will now show how to use any potential-based algorithm for learning nearest neighbor partitions with two centers to construct an algorithm that can learn nearest neighbor partitions with any number of centers. 

Our main result is the following:
\begin{theorem}
\label{thm:two-to-many}
Let $\A$ be a potential-based algorithm for learning nearest neighbor partitions with two centers that has an initial potential (and thus a total loss) of at most $R$. Then there exists a randomized algorithm $\mathcal{A}'$ for learning nearest neighbor partitions for any $k \ge 2$ whose total expected loss is at most $O(k^2R)$. 
\end{theorem}

Similar to many existing methods for multiclass classification (``all-to-all'' methods), we will accomplish this by running one instance of our two-center algorithm $\A$ for each of the $\binom{k}{2}$ pairs of centers. However, instead of using a simple majority voting scheme to choose our eventual label, we will use the \textit{potentials} of these $\binom{k}{2}$ algorithms to construct a distribution over centers that we will sample from. 

More specifically, our algorithm will work as follows. As mentioned, each round, based on the current query $q_t$ and the potentials of the $\binom{k}{2}$ sub-algorithms, we will construct a distribution $v \in \Delta([k])$ over the $k$ centers (we will describe how we do this shortly). We then sample a label $i$ from this distribution $v$ and guess it as the label of $q_t$. If we then learn that the correct label was in fact $j$, we update the sub-algorithm for the pair $(i, j)$ with this information. We do not update any of the other sub-algorithms (in particular, if we guess the label correctly, we do not update any of the sub-algorithms). 

To construct our distribution $v$, we will choose a distribution $v$ that has the property that our expected loss in each round is at most the expected decrease in the total potential over all $\binom{k}{2}$ sub-algorithms. This will guarantee that the total expected loss of our algorithm is bounded above by the total starting potential of all our $\binom{k}{2}$ sub-algorithms. To be more precise, define the following variables:

\begin{enumerate}
    \item  Let $\SubAlg_{ij}$ denote the two-center sub-algorithm for the labels $i$ and $j$. Let $S_{ij}^{(t)}$ be the state of $\SubAlg_{ij}$ at round $t$, and let $\Potential_{ij}^{(t)} = \Potential(S_{ij}^{(t)})$ be the potential of $\SubAlg_{ij}$ at round $t$. Define $\Potential^{(t)} = \sum_{(i, j)}\Potential_{ij}^{(t)}$ to be the total of all the potentials belonging to sub-algorithms $\SubAlg_{ij}$. 
    \item As in Definition \ref{defn:potential-algorithm}, let $L_{ij}^{(t)} = L(S_{ij}^{(t)}, q_t)$ denote an upper-bound on the loss incurred by the algorithm $\SubAlg_{ij}$ in round $t$. 
    
    \item  Let $D_{ij}^{(t)}$ denote the reduction in the potential of $\SubAlg_{ij}$ when $i$ is the correct label. In other words, $D_{ij}^{(t)} = \Potential_{ij}^{(t)} - \Potential_{ij}^{(t+1)}$ when $I^{*}_t=i$. Note that $D_{ij}^{(t)}$ is \textit{not} equal to $D_{ji}^{(t)}$; it is possible for the potential of $\SubAlg_{ij}$ to decrease a lot more upon learning that a point $q_t$ has label $i$ than learning it has label $j$ (geometrically, this corresponds to different halves of the knowledge set being maintained by $\SubAlg_{ij}$).
    \item Finally, define $M^{(t)} \triangleq D^{(t)} - \frac12 \cdot L^{(t)}$, where here we are treating $M^{(t)}$ and $L^{(t)}$ as $n$-by-$n$ matrices. Observe that we can efficiently compute the values $L_{ij}^{(t)}$ and $D_{ij}^{(t)}$ and hence the value of $M_{ij}^{(t)}$ from $q_t$ and the knowledge set of $\SubAlg_{ij}$ at the beginning of round $t$.
\end{enumerate}

From here on, we will fix a round $t$ and suppress all associated superscripts. Assume that in this round the correct label for the query point is $r$. Now, if we sample a label from a distribution $v$, then note that the expected loss we sustain is $\sum_{i}L_{ri}v_i = e_{r}^{T}Lv$. Similarly, the expected decrease in $\Potential$ is $\sum_{i}D_{ri}v_i = e_{r}^{T}Dv$. If it was guaranteed to be the case that $e_{r}^{T}Dv \geq e_{r}^{T}Lv$, then this would in turn guarantee that our total expected loss is at most the total starting potential. 


It follows that if we can find a distribution $v$ that satisfies $Mv \geq 0$, we are in luck. If such a distribution exists, we can find it by solving an LP. This is in fact how we find the distribution $v$, and this concludes the description of the algorithm. To prove correctness of the algorithm, it suffices to show that such a distribution always exists.

To do so, note that from the third point in Definition \ref{defn:potential-algorithm}, we know that for each pair of labels $(i, j)$,
\begin{align}
    D_{ij} + D_{ji} \geq L_{ij}   \tag{P1} \label{P1}
\end{align}

(In fact, Definition \ref{defn:potential-algorithm} tells us that $\max(D_{ij}, D_{ji}) \geq L_{ij}$, since if $\SubAlg_{ij}$ predicts $j$, we have that $D_{ij} \geq L_{ij}$, and likewise if $\SubAlg_{ij}$ predicts $i$, we have that $D_{ji} \geq L_{ij}$). We can rewrite \eqref{P1} in the form $M + M^{T} \geq 0$. The following lemma shows that if $M + M^{T} \geq 0$, then there must exist a distribution $v$ satisfying $Mv \geq 0$, whose proof we defer to the appendix.

\begin{lemma}
\label{lem:distribution-exists}
    Given any matrix $M \in \R^{n\times n}$ such that $M + M^T \geq 0$, there exists 
    a point $v \in \Delta_n$ such that $Mv \geq 0$. 
\end{lemma}
\begin{proof}
To show the existence of such a distribution $v$, we will show that 
the following linear program has a solution. 
\begin{align*}
 \min 0 \\
   Mv &\ge 0 \\
   \sum_i v_i &= 1 \\
  v &\geq 0
\end{align*}
If the above program has no solution then by the Strong Duality Theorem (see~\cite{matousek2007understanding} section 6.1), we  know that the dual program below is unbounded. 
(Since the dual has a feasible solution $y=0,z=0$, we know that it must be unbounded).  In particular, for any value $z>0$, there exists a corresponding solution to $y$ to the dual program.
\begin{align*}
\max z  \\
    M^T y  + z\mathbf{1} &\leq 0 \\
    y &\geq 0
\end{align*}
Let $y$ be any solution to the above dual with $z=1$. Since  $M + M^T$ has all non-negative entries, we know that $(M+M^T)y \geq 0$. Combining this with the fact that $M^Ty \leq -\mathbf{1}$, we get that $My \geq 0$. This contradicts the fact that there is no solution to the primal since $\frac{y}{\sum_{i} y_i}$ is a feasible point. 
\end{proof}

With this, it is straightforward to finish off the proof of Theorem \ref{thm:two-to-many}.

\begin{proof}[Proof of~\Cref{thm:two-to-many}]
By Lemma \ref{lem:distribution-exists}, we know that each round we can find a distribution $v \in \Delta([k])$ satisfying $Mv \geq 0$. By the previous discussion, it follows that if we always sample from this distribution, the total expected loss will be at most twice the starting total potential, $2\Potential^{(1)}$. But note that $\Potential^{(1)}$ is just the sum of the starting potentials $\Potential^{(1)}_{ij}$ of all the sub-algorithms, and is thus at most $\binom{k}{2}R$. It follows that the total loss of our new algorithm is at most $O(k^2R)$. 
\end{proof}

\begin{corollary}\label{cor:k-point}
Fix an $\alpha > 0$. There exists a randomized algorithm for learning nearest neighbor partitions with the inner-product similarity metric that incurs total loss at most $O(\alpha^{-2}k^2d\log d)$.
\end{corollary}

\begin{remark}
Why do simple algorithms (such as a majority voting scheme) fail to work in our setting? In fact, it is possible to get a simple majority vote (breaking ties arbitrarily) to work if we are given additional feedback from the algorithm -- specifically, the ranking of all $k$ distances $\Loss(q_t, \prototype{i})$. With this information, it is possible to update all $\binom{k}{2}$ sub-instances each round, and charge any regret we sustain to an appropriate sub-instance. But if we only receive the true label of $q_t$, we no longer have the information to update every sub-instance, and instead have to do the more subtle amortization described above.

\end{remark}





%% file: general-norm.tex
\section{Learning Nearest Neighbor Partitions Under $L^p$ Similarity}\label{sec:pnorms}
In this section, we will discuss generalizations of our previous results for the inner-product similarity metric to general $\lp$ spaces.  We will primarily deal with the case when there are only two unknown points ($k = 2$) as the general reduction in Section~\ref{sec:multi-reduction} will allow us to reduce from the $k$-point case to the $2$-point case.
 
The general approach for the algorithms in this section is to apply some sort of kernel mapping so that inequalities of the form $\pnorm{X - \prototype{1}}{p} \le \pnorm{X - \prototype{2}}{p}$ become linear constraints. Once we linearize the problem, we can apply our earlier algorithms for the inner-product similarity along with tools from \cite{liu2020optimal}. 

Similar to the previous sections, we will assume that the hidden points $x_1, x_2$ are in the $L^2$ unit ball $B_d$.  This is equivalent to assuming that the hidden points are in the $L^p$ unit ball (which may be a more natural setting since we are working with $L^p$ distances) up to a $\sqrt{d}$ factor since we can simply rescale the $L^2$-ball to contain the $L^p$-ball.

\subsection{$p$-Norms for Even Integers $p$}\label{sec:special-pnorms}
When $p$ is an even integer, there is a kernel mapping that exactly linearizes the problem.  To see this, note that $|(x-a)|^p = (x- a)^p$ which is a polynomial in $x$ so it suffices to consider the polynomial kernel $(1, x, \dots , x^p)$.  In $d$ dimensions, we can simply apply this kernel map coordinate-wise.  After applying these kernel maps, we will be able to apply Corollary \ref{cor:k-point}.  Our main theorem for even integer $p$ is stated below.

\begin{theorem}\label{thm:special-pnorm-kpoint}
  For even integer $p$, there is an algorithm for learning nearest neighbor partitions under the $L^p$ similarity metric that incurs expected total loss at most $O(p^4d^{(p+1)/p}k^2 \log d) = O(k^2\cdot\poly(p, d))$.
\end{theorem}
The details of the proof are deferred to Appendix \ref{sec:appendix-special-pnorms}.

\begin{remark}
For the special case of $p=2$, the reduction is even more immediate (the kernel map needs only add a single dimension), and obtains a slightly tighter bound of $O(k^2d\log d)$. The reduction for this special case is summarized in Appendix \ref{app:euclidean}.
\end{remark}

\subsection{General $p$-Norms}\label{sec:general-pnorms}

Now we discuss how to deal with general $\lp$ norms.  The main difficulty here is that there is no kernel map that exactly linearizes the problem so instead we will have multiple kernel maps.   These kernel maps approximately linearize the problem at different scales, i.e. they have different output dimensions and as the output dimension grows, the problem can be more closely approximated by a linear one.  When we are given a query point, we choose the scale of approximation that we use based on estimating the maximum possible loss that we can incur.  By balancing the dimensionality and the approximation error to be at the same scale as the maximum possible loss, we may ensure that whenever we guess the label incorrectly, a certain potential function must decrease by an amount comparable to the loss that we incur.  Our main theorem is stated below.

\begin{theorem}\label{thm:general-p-norm-k-point}
Fix a $p > 2$. If all $k$ unknown centers are $\Delta$-separated in $L^p$ distance, there exists an algorithm for learning nearest neighbor partitions under the $\lp$ similarity metric that incurs total loss 
\[
\frac{k^2 \poly(d,p) }{\Delta} \cdot \left(\frac{1}{p-2}\right)^2 \,.
\]
\end{theorem}

The full proof is more complicated than the algorithms in previous sections and requires opening the contextual search black-box and redoing parts of the analysis. Note that in the above theorem, we need the assumption that the unknown centers are $\Delta$-separated, an assumption that was not necessary for even integer $p$.  This is due to the fact that when the true centers are too close together, the dimensionality of the kernels that we need to achieve the necessary approximations are large.  Nevertheless, we believe that the separated centers assumption is realistic for classification problems in practice.  The details of the proof of Theorem \ref{thm:general-p-norm-k-point} are deferred to Appendix \ref{sec:appendix-general-pnorms} of the Supplementary Material.

%% file: general-regions.tex
\section{Learning General Convex Regions: Lower Bound}
\label{sec:general-regions}

In this section, we consider the task of learning general convex regions, and present a construction which shows that any learning algorithm incurs $\Omega(T^{(d-4)/(d-2)})$ error over $T$ rounds, even for only $k = 2$ regions. We give the full proof in Appendix~\ref{apx:general-regions}.

\generalregionstheorem*




%% file: contextual_search.tex
\section{Contextual search}\label{app:contextual_search}

In this appendix, we review the contextual search algorithm and analysis presented in \cite{liu2020optimal}. Our presentation will largely follow that of \cite{liu2020optimal}, with two minor changes: 1. we will demonstrate that the algorithm works for all loss functions of the form $\ell(g_t, \langle x_t, p\rangle)$, where it incurs total loss at most $O(\alpha^{-2} d\log d)$ (Theorem \ref{thm:contextual_search}), and 2. we will present the analysis in a slightly different way that makes it easier for us to construct potential-based algorithms for learning nearest neighbors. 

The algorithm is outlined in Algorithm \ref{alg:contextual_search}. Briefly, the algorithm works as follows. Whenever the algorithm gets a query direction $x_t$ from the adversary, the algorithm looks at the width of the current knowledge set in the direction $x_t$. Based on the size of this width, the algorithm picks an ``expansion parameter'' $z_i$, and chooses a guess $g_t$ so that the hyperplane $\dot{v}{x_t} = g_t$ splits the volume of $K_t + z_i\Ball_d$ in half.

\begin{algorithm}[H]
\caption{\sc{Contextual Search Algorithm (\cite{liu2020optimal})}}
\label{alg:contextual_search}
\begin{algorithmic} 
\State Initialize $K_1 = \Ball_d$ and $z_i = 2^{-i} / (8d)$ for all $i$.

\For {$t$ in $1,2, \dots , T$} 
\State Adversary picks $x_t$.
\State Let $i$ be the largest index such that $\width(K_t; x_t) \leq 2^{-i}$.
\State Submit guess $g_t$ such that $\Vol(\{v \in K_t + z_i \Ball_d \mid \dot{v}{x_t} \geq g_t\}) = \frac{1}{2} \Vol(K_t + z_i \Ball_d)$.
\State Update $K_{t+1}$ based on feedback.
\EndFor
\end{algorithmic}
\end{algorithm}

\begin{theorem}[Restatement of Theorem \ref{thm:contextual_search}]\label{thm:contextual_search_app}
Let $\alpha > 0$. Algorithm \ref{alg:contextual_search} is an algorithm for contextual search with loss function $\ell(g_t, \langle x_t, p\rangle) = |g_t - \langle x_t, p \rangle|^{\alpha}$ that incurs a total loss of at most $O(\alpha^{-2}d\log d)$.
\end{theorem}
\begin{proof}
To prove Theorem \ref{thm:contextual_search}, we will examine the following potential function of the knowledge set at time $t$.

$$\Phi(K_t) = \sum_{i=1}^{\infty} 2^{-\alpha i}\log \frac{\Vol\left(K_t + z_i\Ball_{d}\right)}{\Vol(z_i\Ball_{d} )}.$$

Our goal will be to show that $\Phi(K_t)$ decreases by at least the loss we sustain in each round. This will bound the total loss Algorithm \ref{alg:contextual_search} sustains by at most $\Phi(K_0)$. To do this, we will employ the following lemma from \cite{liu2020optimal}:

\begin{lemma}[Lemma 2.1 in \cite{liu2020optimal}]\label{lem:const_split}
If $i$ is the index chosen at round $t$ in Algorithm \ref{alg:contextual_search}, then $\Vol(K_{t+1} + z_i\Ball_d) \leq \frac{3}{4}\Vol(K_{t} + z_i\Ball_d)$.
\end{lemma}

Note that Lemma \ref{lem:const_split} implies that if $i$ is the index chosen at round $t$, then $\Phi(K_t) - \Phi(K_{t+1}) \geq \left(\log\frac{4}{3}\right) 2^{-\alpha i}$. But also, if $i$ is the chosen index at round $t$, then the width in the query direction is at most $2^{-i}$, and thus the loss sustained in this round is at most $2^{-\alpha i}$. If we let $L_t$ be the loss sustained in round $t$, we have thus shown that 

$$\Phi(K_t) - \Phi(K_{t+1}) \geq \left(\log\frac{4}{3}\right)L_t.$$

Summing this over all $t$, we find the total loss is at most $O(\Phi(K_1))$. Since $K_1 = \Ball_d$, we can evaluate $\Phi(K_1)$ as follows:

\begin{eqnarray*}
\Phi(K_1) &=& \sum_{i=1}^{\infty} 2^{-\alpha i}\log \frac{\Vol\left(\Ball_d + z_i\Ball_{d}\right)}{\Vol(z_i\Ball_{d})} \\
&=& \sum_{i=1}^{\infty} 2^{-\alpha i} d \log\left(1 + \frac{1}{z_i}\right) \\
&\leq& \sum_{i=1}^{\infty} 2^{-\alpha i} d \log\left(2^{i+4}d\right) \\
&=& O\left(d\sum_{i=1}^{\infty} 2^{-\alpha i}i\right) + O\left(d\log d\sum_{i=1}^{\infty} 2^{-\alpha i}\right) \\
&=& O(\alpha^{-2}d) + O(\alpha^{-1}d\log d) \\
&\leq& O(\alpha^{-2}d\log d).
\end{eqnarray*}
\end{proof}

\section{Potential-based algorithm for inner-product similarity}\label{app:potential-based}

In this appendix we prove Theorem \ref{thm:two-point}, showing that the described algorithm for learning nearest neighbor partitions in Section \ref{sec:two-point} is a potential-based algorithm. Indeed, we will be able to use the same potential function as in the proof of Theorem \ref{thm:contextual_search_app}, namely:

$$\Phi(K_t) = \sum_{i=1}^{\infty} 2^{-\alpha i}\log \frac{\Vol\left(K_t + z_i\Ball_{d}\right)}{\Vol(z_i\Ball_{d} )}.$$

Note that $\Phi(K_t)$ clearly satisfies the first two conditions in Definition \ref{defn:potential-algorithm}, namely $\Phi(K_t) \geq 0$ for any knowledge set $K_t$, and it is always the case that $\Phi(K_{t+1}) \leq \Phi(K_t)$ (in particular, since $K_{t+1} \subseteq K_{t}$). It thus suffices to show the third condition of Definition \ref{defn:potential-algorithm} holds:

\begin{lemma}
Let $\A$ be the ($k=2$) algorithm for learning nearest neighbor partitions with similarity metric $\delta(x, y) = -\langle x, y \rangle$ and loss raised to the power $\alpha$ described in Corollary \ref{cor:two-point}. Then, for all rounds $t$ where $\A$ guesses the label incorrectly, $\Phi(K_{t}) - \Phi(K_{t+1}) \geq \Omega\left(\max_{v \in K_t} |\dot{q_t}{v}|^{\alpha}\right)$.
\end{lemma}
\begin{proof}
Recall from the proof of Theorem \ref{thm:csearch_reduction}, whenever $\A$ guesses the label incorrectly, we update the state of the contextual search algorithm underlying $\A$. The contextual search algorithm underlying $\A$ shares the same knowledge set $K_t$ as $\A$, and by the analysis in the proof of Theorem \ref{thm:contextual_search_app}, $\Phi(K_t)$ must then satisfy

$$\Phi(K_t) - \Phi(K_{t+1}) \geq \left(\log\frac{4}{3}\right)\width(K_t; q_t)^{\alpha}.$$

Since $\width(K_t; q_t) = \max_{v \in K_t}\dot{v}{q_t} - \min_{v \in K_t}\dot{v}{q_t}$, $\width(K_t; q_t) \geq \max_{v \in K_t} |\dot{q_t}{v}|$, and we have proved this lemma.
\end{proof}

%% file: apx-euclidean-norm.tex
\section{From Euclidean Distance to Inner-Product Similarity}\label{app:euclidean}


We have presented two different variants of the nearest neighbor partition problem: one where we want to return the point $\prototype{i}$ with largest inner-product similarity to each query $\query{t}$, and one where we want to return the point $\prototype{i}$ closest to the query $\query{t}$ in some $\lp$ norm. Here we will show that in the case of the Euclidean norm, we can easily reduce the second problem to the first -- and therefore, it suffices to solve the problem only for the case of inner-product similarity in Section \ref{sec:inner-product}. More specifically, we will show that if we can solve the nearest neighbor partition problem for the similarity metric $\delta(x, y) = -\inner{x}{y}$ and $\alpha = 1/2$, we can solve the nearest neighbor partition problem for the $\lp[2]$ similarity metric $\delta(x, y) = \pnorm{x-y}{2}$. (In many ways, this can be seen as a warm-up for the more general case of even integer $p$ $\lp$ norms in Appendix \ref{sec:appendix-special-pnorms}). 

Let $\prototype{1}, \prototype{2}, \dots, \prototype{k}$ be points in $\Ball_{d}$. Consider following two maps $T$ and $Q$ from $\Ball_d \rightarrow \Ball_{d+1}$. $T$ maps the point $x \in \Ball_d$ to $T(x) \triangleq \frac{1}{\sqrt{2}}(x, \pnorm{x}{2}^2)$ where $(x, \pnorm{x}{2}^2)$ is the $(d+1)$-dimensional vector formed by appending $\pnorm{x}{2}$ to $x$. $Q$ maps the point $\query{} \in \Ball_d$ to $Q(\query{}) \triangleq \frac{1}{\sqrt{5}}(2 \query{}, -1)$. We now have the following two claims.

\begin{lemma}\label{lem:parabola}
  Let $X = \{\prototype{1}, \prototype{2}, \dots, \prototype{k}\}$ be a set of points in $\Ball_d$. Then for any $q \in \Ball_d$, if \\ $x^* = \arg\min_{x \in X} \pnorm{q-x}{2}$, it is also true that $x^* = \arg\max_{x \in X} \langle T(x), Q(q)\rangle$. 
\end{lemma}

\begin{proof}
Consider two points $x, x' \in \Ball_d$. It suffices to show that if $\norm{q-x} \leq \norm{q-x'}$, then $\inner{T(x)}{Q(q)} \ge \inner{T(x')}{Q(q)}$.

To see this, note that we can rewrite $\norm{q-x}^2 \le \norm{q-x'}^2$ in the form $\inner{q-x}{q-x} \le \inner{q-x'}{q-x'}$, which we can in turn simplify to get
\begin{equation}\label{eq:eqdist}
  -2 \inner{q}{x} + \norm{x}^2 \le -2 \inner{q}{x'} + \norm{x'}^2\text{.}
\end{equation}

But the LHS of \eqref{eq:eqdist} is simply $-\sqrt{10} \inner{T(x)}{Q(q)}$ while the RHS of \eqref{eq:eqdist} is likewise $-\sqrt{10}\inner{T(x')}{Q(q)}$. Equation \eqref{eq:eqdist} thus implies that $\inner{T(x)}{Q(q)} \ge \inner{T(x')}{Q(q)}$, as desired.
\end{proof}

\begin{lemma}\label{lem:parabola_loss}
Let $x$, $x'$, and $q$ be points in $\Domain$. Let
$$\ell_1 \triangleq \abs{\norm{q-x} - \norm{q-x'}}$$
and
$$\ell_2 \triangleq \abs{\inner{T(x)}{Q(q)} - \inner{T(x')}{Q(q)}}.$$

Then $\ell_1 \le 2\sqrt{\ell_2}$.
\end{lemma}

\begin{proof}
Note that $\ell_1 \le \norm{q-x} - \norm{q-x'}$, so in particular
\begin{equation}\label{eq:loss_ineq}
  \ell_1^2 \leq \abs{ \norm{q-x}^2 - \norm{q-x'}^2}.
\end{equation}

Via the same logic in the proof of Lemma \ref{lem:parabola}, we can rewrite the RHS of \eqref{eq:loss_ineq} as $\sqrt{10}\ell_2$. It follows that $\ell_1^2 \le \sqrt{10}\ell_2$ and thus that $\ell_1 \leq 2\sqrt{\ell_2}$.
\end{proof}

With these two lemmas, we can prove the following reduction.

\begin{theorem}\label{thm:euclidean-reduction}
  Let $\A$ be an algorithm for learning nearest-neighbor partitions under the similarity metric $\delta(x, y) = -\inner{x}{y}$ with $\alpha = 1/2$ that achieves a total loss of at most $R(k, d)$. Then there exists an algorithm $\A'$ for learning nearest-neighbor partitions under the similarity metric $\delta(x, y) = \pnorm{x-y}{2}$ (and $\alpha = 1$) that achieves a total loss of at most $2 R(k, d+1)$.
\end{theorem}
\begin{proof}
To construct algorithm $\A'$ from algorithm $\A$, we simply map each incoming query $q_t \in \Ball_{d}$ for algorithm $\A'$ to the point $q'_t = Q(q_t) \in \Ball_{d+1}$ and feed it to $\A$ (returning the label that $\A$ outputs, and providing $\A$ with the true label that we receive). 

To see why this works, note that if the hidden centers for $\A'$ are the points $\prototype{1}, \prototype{2}, \dots, \prototype{k} \in \Ball_d$, then by Lemma \ref{lem:parabola}, all feedback we provide $\A$ is consistent with the set of hidden centers $T(\prototype{1}), T(\prototype{2}), \dots, T(\prototype{k}) \in \Ball_{d+1}$. Moreover, by Lemma \ref{lem:parabola_loss}, whenever algorithm $\A$ incurs loss $\ell$, our algorithm $\A'$ incurs loss at most $2\ell$. It follows that $\A'$ incurs loss at most $2R(k, d+1)$. 
\end{proof}

%% file: apx-mistake-bounds.tex
\section{Mistake Bounds}

In this section we provide mistake bounds for our algorithms for learning linear classifiers and learning nearest neighbor partitions. In both cases we will get (near) state-of-the-art guarantees for the mistake bound, despite our algorithms being designed for the absolute loss function. 

We begin by discussing our algorithm for learning linear classifiers. As noted in the introduction, note that since this algorithm incurs total loss of at most $O(d \log d)$, then we make at most $O(d \log d / \gamma)$ mistakes in a setting with margin $\gamma$. In the following theorem, we see that we can improve this bound to $O(d \log 1/\gamma + d\log d)$ (matching the mistake bound of the best halving-based algorithms whenever $\gamma \leq 1/d$). 

\begin{theorem}\label{thm:mistake1}
Assume every query point $q_t$ we are provided satisfies $|\langle q_t, x_1 - x_2\rangle| \geq \gamma$ for some $\gamma > 0$. Then the algorithm of Theorem \ref{thm:intro1} makes at most $O(d\log 1/\gamma + d\log d)$ mistakes.
\end{theorem}

\begin{proof}
The main observation is that since we query our contextual search subroutine (which is trying to learn the hidden point $w = x_1 - x_2$) with the point $q_t$, if $|\langle q_t, x_1 - x_2 \rangle| \geq \gamma$, then either 1. we already know for certain the sign of $\langle q_t, x_1 - x_2 \rangle$, or 2. the width $\width(K_t; q_t) \geq \gamma$. 

In the first case, we cannot make a mistake. In the second case, since the width is at least $\gamma$, it suffices to only consider the $\lceil\log(1/\gamma)\rceil$th term of the potential function in Theorem \ref{thm:contextual_search_app}. Specifically, note that by Lemma \ref{lem:const_split}, if we let $i = \lceil\log(1/\gamma)\rceil$ then we have that

$$\Vol(K_{t+1} + z_{i}B_d) \leq \frac{3}{4}\Vol(K_t + z_iB_d).$$

\noindent
In particular, since $z_i = 2^{-i}/8d$, our total number of errors is at most

$$\log_{4/3} \frac{\Vol\left(\Ball_d + z_i\Ball_{d}\right)}{\Vol(z_i\Ball_{d})} = d \log_{4/3}\left(1 + \frac{1}{z_i}\right) = O(d \log (1/\gamma) + d\log d).$$
\end{proof}

The same logic extends to learning nearest-neighbor partitions via the reduction in Theorem \ref{thm:two-to-many}.

\begin{theorem}\label{thm:mistake2}
Assume every query point $q_t$ satisfies $\delta(q_t, R_i) > \gamma$ for all $i$ such that $q_t \not\in R_i$. Then the algorithm of Theorem \ref{thm:intro2} makes at most $O(k^2d(\log 1/\gamma + \log d))$ mistakes.
\end{theorem}
\begin{proof}
We apply the reduction of Theorem \ref{thm:two-to-many} to Theorem \ref{thm:mistake1}. In particular, note that the analysis of Theorem \ref{thm:mistake1} implies that our original algorithm for learning linear classifiers can be thought of (under these margin conditions) as a potential-based algorithm with loss function $L(S_t, q_t) = 1$ and with potential function

$$\Phi(S_t) = \log_{4/3} \frac{\Vol\left((1 + z_i)\Ball_{d}\right)}{\Vol(S_t + z_i\Ball_{d})}.$$

The analysis of Theorem \ref{thm:mistake1} combined with the guarantees of Theorem \ref{thm:two-to-many} imply a mistake bound of $O(k^2d(\log 1/\gamma + \log d))$.
\end{proof}

Finally, we prove a general reduction for the notion of robust mistake bound defined in the introduction. Formally, the robust mistake bound with margin $\gamma$ is the loss induced by the loss function $\ell'(q, R_i) = \mathbf{1}(\delta(q, R_i) - \delta(q, R^{*}) \geq \gamma)$ (where $R^{*}$ is the region containing $q$). In the below lemma, we relate this to the loss induced by our standard loss function $\ell(q, R_i) = (\delta(q, R_i) - \delta(q, R^{*}))$.

\begin{lemma}
If an algorithm has total loss at most $R$ under the loss function  $\ell(q, R_i) = (\delta(q, R_i) - \delta(q, R^{*}))$, it has a robust mistake bound of $O(R/\gamma)$ under margin $\gamma$.
\end{lemma}
\begin{proof}
This immediately follows form the fact that:

$$\ell(q, R_i) = (\delta(q, R_i) - \delta(q, R^{*})) \geq \gamma \cdot \mathbf{1}(\delta(q, R_i) - \delta(q, R^{*}) \geq \gamma) = \gamma \ell'(q, R_i).$$
\end{proof}

%% file: apx-general-norm.tex
\section{Omitted Proofs }\label{app:pnorms}

\subsection{Omitted Proofs from Section \ref{sec:special-pnorms}}\label{sec:appendix-special-pnorms}
Here we prove Theorem \ref{thm:special-pnorm-kpoint}.  Fix an even integer $p$. We define the following kernel maps.
\begin{definition}
  Let $\Ker: \R \rightarrow \R^{p+1}$ be the map defined by 
  \[
    \Ker(x) \triangleq (1, x, x^2, \dots, x^{p}) \,.
  \]
\end{definition}
\begin{definition}
  For a point $y = (y_1, \dots, y_d) \in \Ball_d$, define the map $G:\Ball_d \rightarrow \Ball_{(p+1)d}$ as
  \begin{align*}
    G(y) \triangleq \frac{1}{\sqrt{pd}}\left(\Ker\left(\frac{y_1}{p}\right), \dots , \Ker\left(\frac{y_d}{p}\right)\right)
  \end{align*}
  where above the outputs of $\Ker( \cdot )$ are simply concatenated.
\end{definition}

\begin{definition}
  Let $F: \R \rightarrow \R^{p+1}$ be defined by
  \[
    F(a) \triangleq \left( a^p, -\binom{p}{1} a^{p-1}, \dots , -\binom{p}{p-1}a, \binom{p}{p}     \right)
  \]
\end{definition}

Note that if $\abs{a} \le 1/p$, all components of $F(a)$ are at most $1$ in absolute value.

\begin{definition}
  For a point $z = (z_1, \dots, z_d) \in \R^d$, define the map $H:\Ball_d \rightarrow \Ball_{(p+1)d}$ as
  \begin{align*}
    H(z) \triangleq \frac{1}{\sqrt{pd}}\left(F\left(\frac{z_1}{p}\right), \dots , F\left(\frac{z_d}{p}\right)\right) \,.
  \end{align*}
\end{definition}

The key property that these maps satisfy is stated below.
\begin{lemma}\label{lem:kernel-approx-basic}
  For points $y,z \in \R^d$, 
  \[
    \inner{G(y)}{H(z)} = \frac{1}{pd}\pnorm{\frac{y-z}{p}}{p}^p \,.
  \]
\end{lemma}

\begin{proof}
  The proof follows by substituting in the definitions for $G$ and $H$ and using the binomial theorem. 
\end{proof}

In particular, we can rewrite the statement of Lemma \ref{lem:kernel-approx-basic} in the form

\begin{equation}\label{eq:pnorm-to-ip}
    \pnorm{y-z}{p} = p^{(p+1)/p}d^{1/p} \inner{G(y)}{H(z)}^{1/p}.
\end{equation}

This suggests a reduction to the inner-product similarity metric similar to the reduction for Euclidean norm in Theorem \ref{thm:euclidean-reduction}. In particular, note that for $p \geq 1$, we have that:
$$|x|^{p} - |y|^{p} \geq |x - y|^{p}.$$

In particular, this implies that 
$$||q - x|| - ||q - x^*|| \leq (||q - x||^p - ||q - x^*||^p)^{1/p} = p^{(p+1)/p}d^{1/p}\left|\inner{G(q)}{H(x^*)} - \inner{G(q)}{H(x)}\right|^{1/p}.$$

Thus, if we use our map $H$ to map each query point $q_t$ to the point $H(q_t) \in \Ball_{p(d+1)}$ and feed it into an algorithm with similarity metric $\Loss(x, y) = -\langle x, y\rangle$ and $\alpha = 1/p$, Lemma \ref{lem:kernel-approx-basic} implies that this algorithm will successfully learn the partition induced by the points $G(\prototype{i})$.  We can now complete the proof of Theorem \ref{thm:special-pnorm-kpoint}.

\begin{proof}[Proof of Theorem \ref{thm:special-pnorm-kpoint}]
  From Corollary \ref{cor:k-point}, there exists an algorithm for learning nearest neighbor partitions with similarity function $\delta(x, y) = -\inner{x}{y}$ and $\alpha = 1/p$ with expected total loss $O(p^2k^2d\log d)$. Applying this algorithm as described above (noting that the ambient dimension is now $p(d+1)$ and the loss is scaled by a factor of $O(pd^{1/p})$), we obtain the bound in the statement.
\end{proof}

\subsection{Omitted Proofs from Section \ref{sec:general-pnorms}}\label{sec:appendix-general-pnorms}
Here, we prove Theorem \ref{thm:general-p-norm-k-point}.

\subsubsection{Kernelization}

Fix the $\lp$ norm that we are working with.  Let $p' = \lfloor p \rfloor + 1$.  For each $i = 1, 2, \dots$, let 
\[
  \delta_i = \frac{1}{100d^2 p' 2^i}, D_i = \frac{1}{2\delta_i} \,.
\]
We will now define two maps that will be crucial for our algorithm.

For each $i$, we will define two maps $G_i, H_i$ that map points in $\R^d$ to points in $\R^{p'd(2D_i+1)}$ such that for two points $y,z \in \R^d$, the images $G_i(y), H_i(z)$ satisfy the property that $\inner{G_i(y)}{H_i(z)}$ is a good approximation of $\pnorm{y - z}{p}^p$.  We need to consider maps for different values of $i$ because as $i$ increases, the approximation gets better but the dimension of the image space also increases.  We can think of these maps for different values of $i$ as approximations at different scales.

Through the next several definitions, we build the first map $G_i$.
\begin{definition}
  For $x \neq 0$, let $\sign(x) \triangleq x/\abs{x}$. Let $\sign(0) \triangleq 0$.
\end{definition}

\begin{definition}
  Let $D: \R \rightarrow \R^{p'}$ be the map defined by 
  \[
    D(x)  = (\abs{x}^p, \sign(x)\abs{x}^{p-1}, \abs{x}^{p-2}, \sign(x)\abs{x}^{p-3}, \dots , \sign(x)^{\lfloor p \rfloor} |x|^{p - \lfloor p \rfloor}) \,.
  \]
\end{definition}
\begin{definition}
Let $\Ker_i : \R \rightarrow \R^{p'(2D_i+1)}$ be the map defined by 
\[
\Ker_i(x) = \left( D(x + 0.5), D(x + 0.5 - \delta_i) , \dots , D(x - 0.5 ) \right)
\]
where the tuples given by the output of $D(\cdot )$ are simply concatenated.
\end{definition}

\begin{definition}
For a point $y = (y_1, \dots, y_d) \in B_d$, define the map $G_i:\R^d \rightarrow \R^{p' d(2D_i+1)}$ as
\begin{align*}
G_i(y) = (\Ker_i(y_1/2), \dots , \Ker_i(y_d/2)) \,.
\end{align*}
\end{definition}

Now we build the second map $H_i$ through the next set of definitions.
\begin{definition}\label{def:taylor-map}
Let $F_i: [-1/2, 1/2] \rightarrow \R^{p' (2D_i + 1)} $ be the map defined as follows.  Assume that we want to compute $F_i(x)$.  Then perform the following steps.
\begin{itemize}
    \item Let $c$ be the unique integer such that $c \delta_i \leq x < (c+1)\delta_i $
    \item  For any element of $\R^{p'(2D_i+1)}$, group the coordinates into consecutive groups of $p'$ and label the groups with $-D_i, -D_i + 1, \dots , D_i$
    \item Let $F_i(x)$ be the element of $\R^{p'(2D_i+1)}$ where
    \begin{itemize}
        \item The group labeled $c$ is set to
        \[
        \left(1, p(x - c\delta_i), \frac{p(p-1)}{2}(x - c\delta_i)^2, \frac{p(p-1)(p-2)}{6}(x - c\delta_i)^3, \dots , \right)
        \]
        \item  All other groups are set to $(0,0, \dots , 0)$
    \end{itemize} 
\end{itemize}
\end{definition}

\begin{definition}
For a point $z = (z_1, \dots, z_d) \in B_d$, define the map $H_i:\R^d \rightarrow \R^{p' d(2D_i+1)}$ as
\begin{align*}
H_i(z) = (F_i(z_1/2), \dots , F_i(z_d/2)) \,.
\end{align*}
\end{definition}

The intuition for the interplay between the maps $G_i$ and $H_i$ is that the ``kernel" map $G_i$ discretizes the function $|x|^p$ as well as its derivatives and then $H_i$ takes the first $p'$ terms of the Taylor series expansion at the closest point in the discretization.

Formally, the key property that the maps $G_i,H_i$ satisfy is the following:
\begin{lemma}\label{lem:kernel-approx}
For points $y,z \in B_d$,
\[
\left\lvert \langle G_i(y), H_i(z) \rangle - \norm{\frac{y - z}{2}}_p^p \right \rvert \leq d (p\delta_i)^{p}
\]
\end{lemma}
The proof of Lemma \ref{lem:kernel-approx} relies on the following inequality.

\begin{claim}\label{claim:approx-norm}
Let $p>2$.  Then for any $x, x' \in [-1,1]$, we have the inequality
\[
\left\lvert \abs{x}^p -  \sum_{i=0}^{\lfloor p \rfloor}  \frac{p(p-1) \dots (p-i + 1)}{i!}(x - x')^i\sign(x')^{i}\abs{x'}^{p-i}  \right\lvert \leq (p|x - x'|)^{p} \,.
\]
\end{claim}
\begin{proof}
Note that the function $\abs{x}^p$ is $\lfloor p \rfloor$-times continuously differentiable and its derivatives are are $p \cdot \sign(x)\abs{x}^{p-1}, p(p-1)\abs{x}^{p-2}, \dots $ and so on.  Thus, we may write
\begin{align*}
\abs{x}^p &= \abs{x'}^p + \int_{x'}^xp \cdot \sign(y)\abs{y}^{p-1}  dy \\ &=  \abs{x'}^p +  p(x - x') \sign(x') \abs{x'}^{p-1} + \int_{x'}^x \int_{x'}^{y_1} p(p-1)\abs{y_2}^{p-2} dy_2 dy_1 \\  & \vdots \\ &= \sum_{i=0}^{\lfloor p \rfloor }  \frac{p(p-1) \dots (p-i + 1)}{i!}(x - x')^i\sign(x')^{i}\abs{x'}^{p-i} \\ & \quad + \int_{x'}^x \dots  \int_{x'}^{y_{\lfloor p \rfloor - 1}} p \cdots (p - \lfloor p \rfloor + 1) (\sign(y_{\lfloor p \rfloor})^{\lfloor p  \rfloor}\abs{y_{\lfloor p \rfloor}}^{p-\lfloor p \rfloor  } - \sign(x')^{\lfloor p  \rfloor}\abs{x'}^{p-\lfloor p \rfloor})  dy_{\lfloor p \rfloor} \dots dy_1\,.
\end{align*}
It now suffices to bound the last term which is the ``error" term..  However since $p - \lfloor p \rfloor < 1$,
\[
\left \lvert \sign(y)^{\lfloor p  \rfloor}\abs{y}^{p-\lfloor p \rfloor  } - \sign(x')^{\lfloor p  \rfloor}\abs{x'}^{p-\lfloor p \rfloor} \right \rvert \leq \abs{y - x'}^{p - \lfloor p \rfloor} \,,
\]
so the error term is at most 
\[
 p \cdots (p - \lfloor p \rfloor + 1) |x - x'|^p \leq (p|x- x'|)^p\,,
\]
and now we immediately get the desired inequality.
\end{proof}

Now we can prove Lemma \ref{lem:kernel-approx}.
\begin{proof}[Proof of Lemma \ref{lem:kernel-approx}]
Let $y = (y_1, \dots , y_d)$ and $z = (z_1, \dots , z_d)$.  For each $j \in [d]$ let $c_j$ be the integer such that $c_j \delta_j \leq 0.5z_j < (c_j+1)\delta_i$.  Now we have
\begin{align*}
\langle G_i(y), H_i(z) \rangle = \sum_{j \in [d]} \bigg( \sum_{i=0}^{\lfloor p \rfloor}\frac{p(p-1) \dots (p-i + 1)}{i!}(0.5z_j - c_j \delta_i)^{i} \sign(0.5y_j - c_j \delta_i)^i\abs{0.5y_j - c_j \delta_i}^{p-i} \bigg)\,.
\end{align*}
Now we can apply Claim \ref{claim:approx-norm} with $x = 0.5(y_j - z_j)$ and $x' = 0.5 y_j - c_j\delta_i$ to bound each term.  Note that $|x - x'| < \delta_i$.  Since the sum contains $d$ terms, we immediately get the desired conclusion. 
\end{proof}

Our full algorithm for learning nearest neighbor partitions in $\lp$ norm is described below.

\subsubsection{Algorithm}

\begin{algorithm}[H]
\caption{{\sc Multiscale Nearest Neighbor Learning for $\lp$ norms} }
\begin{algorithmic}
\State There are two unknown points $x_1,x_2 \in  \Domain$
\State For each $i = 1,2, \dots $ initialize the sets $S_i = [-1,1]^{2p'd(2D_i + 1)}$.  Note that
\[
S_i \supset \{ (G_i(x), G_i(y))| x,y \in [-1/2, 1/2]^d \} 
\]
\For {$t$ in $1,2, \dots , T$}
\State Adversary picks $q_t \in B_d$
\For {$i = 1,2, \dots $}
\State Let $v_{i,t} = (-H_i(q_t), H_i(q_t))$
\State Let $w_{i,t}$ be the width of the set $S_i$ in direction $v_{i,t}$ i.e.
\[
w_{i,t} = \max_{u \in S_i}\langle v_{i,t}, u \rangle  - \min_{u \in S_i }\langle v_{i,t}, u \rangle
\]
\State Let $i_t$ be the smallest integer $i$ such that $w_{i,t} \geq 10^3D_ipd^2 (p\delta_i)^p$
\EndFor
\If{$\Vol\left( \{ z \in S_{i_t}, v_{i_t,t} \cdot z > 0 \}\right) \geq \Vol\left( \{ z \in S_{i_t}, v_{i_t,t} \cdot z < 0 \}\right)$}
\State Guess label $1$ 
\EndIf
\If{$\Vol\left( \{ z \in S_{i_t}, v_{i_t,t} \cdot z > 0 \}\right) < \Vol\left( \{ z \in S_{i_t}, v_{i_t,t} \cdot z < 0 \}\right)$}
\State Guess label $2$ 
\EndIf
\If {true label is $1$} update
\[
S_{i_t} \leftarrow \{ z \in S_{i_t}, v_{i_t,t} \cdot z > -3d (p\delta_{i_t})^{p} \}
\]
\EndIf
\If {true label is $2$} update
\[
S_{i_t} \leftarrow \{ z \in S_{i_t}, v_{i_t,t} \cdot z < 3d(p\delta_{i_t})^{p} \}
\]
\EndIf
\EndFor
\end{algorithmic}
\end{algorithm}
\begin{remark}
In the algorithm, it is stated that we keep track of sets $S_i$ for all integers $i$.  Technically, this is not possible as there are infinitely many sets to keep track of but it will be clear from the analysis that it suffices to track the set $S_i$ only for $i \leq \poly(p,d,T)$ and if $i_t$ is too large, we can simply guess arbitrarily and our loss will be upper bounded by $1/T$. 
\end{remark}

As in the previous section, for each timestep $t$ and integer $i$, we let $S_{i}^{(t)}$ denote the set $S_i$ at the beginning of timestep $t$ in the execution of the algorithm.

\begin{claim}\label{claim:contains-ball}
For all $i$ and all timesteps $t$, the set $S_i^{(t)}$ contains the $L^2$ ball of radius $0.1(p\delta_i)^{p}$ centered around $(G_i(x_1), G_i(x_2))$ where $x_1, x_2$ are the two unknown centers. 
\end{claim}
\begin{proof}
We will prove the claim by induction on $t$.  The base case is obvious.  Now we do the induction step.  Consider a timestep $t$.  Assume that the adversary gives us the point $q_t$.  WLOG the true label of $q_t$ is $1$ i.e. 
\[
\norm{q_t - x_1}_p^p \leq \norm{q_t - x_2}_p^p \,.
\]
Note that $\norm{v_{i,t}}_2 \leq 10d$ for all $i,t$ (this uses the fact that in Definition \ref{def:taylor-map}, $|x - c\delta_i| \leq \delta_i \leq 1/p$).  Thus, any point $z$ in the ball of radius $0.1(p\delta_i)^{p}$ centered around $(G_i(A), G_i(B))$ satisfies 
\[
  \abs{v_{i,t} \cdot z  - v_{i,t} \cdot (G_i(x_1), G_i(x_2))} \le d(p\delta_i)^{p} \,.
\]
However
\[
  v_{i,t} \cdot (G_i(x_1), G_i(x_2)) = -\inner{H_i(q_t)}{G_i(x_1)}  + \inner{H_i(q_t)}{G_i(x_2)}
\]
and we can now use Lemma \ref{lem:kernel-approx} to deduce
\[
  \abs{\left(-\pnorm{\frac{q_t - x_1}{2}}{p}^p + \pnorm{\frac{q_t - x_2}{2}}{p}^p\right) - v_{i,t} \cdot (G_i(x_1), G_i(x_2))} \le  2d(p\delta_i)^{p} \,.
\]
Thus, by the triangle inequality, we must actually have for all $z$ in the ball of radius $0.1(p\delta_i)^{p}$ centered around $(G_i(x_1), G_i(x_2))$,
\[
  v_{i,t} \cdot z \geq -  3d(p\delta_i)^{p}
\]
which implies that all of these points are all contained in $S_i$ after the update step, completing the induction.
\end{proof}

We will also need the following geometric fact.
\begin{lemma}[From \cite{liu2020optimal}]\label{lem:strip-volume}
  Let $S \subset \R^d$ be a convex polytope and $v$ be a unit vector. Assume that the width of $S$ in direction $v$ is at least $8d\eps$. Then any strip of width $\eps$ normal to direction $v$ contains at most $1/4$ of the volume of $S$.
\end{lemma}
\begin{proof}
  Let $C$ be a cross section of $S$ normal to direction $v$ with maximal area.  Let $u_1,u_2$ be two points in $S$ that minimize and maximize the inner product with $v$ respectively.  Then either $u_1$ or $u_2$ is distance at least $4d\eps$ from the hyperplane containing $C$.  WLOG $u_1$ is at least $4d\eps$ away from this hyperplane.  Then the cone containing $u_1$ and $C$ must be contained in $S$ (since $S$ is convex) so
  \[
    \Vol(S) \geq 4\eps d \cdot \Vol(C) \cdot \frac{1}{d} = 4\eps \Vol{C} \,.
  \]
  On the other hand, by the maximality of $C$, the volume contained in any $\eps$-width strip is at most $\eps \Vol(C)$ so we are done.
\end{proof}

\begin{claim}\label{claim:volume-decrease}
Consider a timestep $t$.  If the learner incurs nonzero loss then 
\[
  \Vol\left(S_{i_t}^{(t+1)}\right)
    \le \frac34 \Vol\left(S_{i_t}^{(t)}\right) \,. 
\]
\end{claim}

\begin{proof}
Without loss of generality, the true label is $1$ and our guess was $2$. Then we must have  
\[
  \Vol\left( \{ z \in S_{i_t}^{(t)}, v_{i_t,t} \cdot z > 0 \}\right) \le \frac12 \Vol\left(S_{i_t}^{(t)}\right) \,.
\]
Also since $w_{i,t} \geq 10^3D_{i_t}pd^2 (p\delta_{i_t})^{p}$ and the dimension of the space in which $S_{i_t}^{(t)}$ lives is $2p'd(2D_{i_t} + 1)$, Lemma \ref{lem:strip-volume} gives us that   
\[
\Vol\left( \{ z \in S_{i_t}^{(t)}, -3d(p\delta_{i_t})^{p} < v_{i_t,t} \cdot z < 0 \}\right) \le \frac14 \Vol\left(S_{i_t}^{(t)}\right) \,.
\]
Thus, we deduce that 
\[
  \Vol(S_{i_t}^{(t+1)}) = \Vol\left( \{ z \in S_{i_t}^{(t)}, v_{i_t,t} \cdot z  > -3d(p\delta_{i_t})^{p}  \}\right) \le \frac34 \Vol\left(S_{i_t}^{(t)}\right) \,,
\]
as desired.
\end{proof}

\begin{theorem}\label{thm:general-p-norm}
The total loss incurred by {\sc Multiscale Nearest Neighbor Learning} is at most
\[
\frac{\poly(d,p)  }{\norm{x_1-x_2}_p} \cdot \left(\frac{1}{p-2}\right)^2 \,.
\]
\end{theorem}
\begin{proof}
Combining Claim \ref{claim:contains-ball} and Claim \ref{claim:volume-decrease} implies that for any index $i$, the number of times that $i_t = i$ and we incur nonzero loss is at most
\[
O \left(\log \left( \frac{2^{2p'd(2D_i + 1)} }{ \left(0.1(p\delta_i)^{p}\right)^{2p'd(2D_i + 1)}\Vol(B_{2p'(2D_i + 1)}(0,1)) } \right) \right)  = O\left(  p^3d^32^i (i + \log dp)  \right) \,.
\]
Note that if $i_t = i$, then by the definition of our algorithm
\[
w_{i-1, t}  = \max_{u \in S_{i-1}^{t}}\langle v_{i-1,t}, u \rangle  - \min_{u \in S_{i-1}^{t} }\langle v_{i-1,t}, u \rangle \leq  10^3D_{i-1}pd^2 (p\delta_{i-1})^{p} \,.
\]
Note that the origin is clearly always contained in $S_{i-1}$ so using Claim \ref{claim:contains-ball}, we get that 
\[
 \left \lvert \langle v_{i-1,t}, (G_{i-1}(x_1), G_{i-1}(x_2)) \rangle \right \rvert \leq  10^3D_{i-1}pd^2 (p\delta_{i-1})^{p} = O(d^2p^2(p\delta_{i-1})^{p-1}) \,.
\]
Lemma \ref{lem:kernel-approx} implies that
\[
\left \lvert  \left(-\norm{\frac{q_t - x_1}{2}}_p^p + \norm{\frac{q_t - x_2}{2}}_p^p\right) - v_{i-1,t} \cdot (G_{i-1}(x_1), G_{i-1}(x_2))\right \rvert \leq  2d (p\delta_{i-1})^{p}  \,,
\]
so we deduce that 
\[
\left \lvert  \norm{q_t - x_1}_p^p - \norm{q_t - x_2}_p^p\right \rvert  \leq  O(d^2p^2(4 p\delta_{i-1})^{p-1}) \leq O( d^2p^2 ( 10 \cdot 2^id^2)^{-(p-1)})\,.
\]
Note that our loss at each round may be bounded as 
\[
\left \lvert \norm{q_t - x_1}_p - \norm{q_t - x_2}_p \right \rvert \leq \frac{\left \lvert  \norm{q_t - x_1}_p^p - \norm{q_t - x_2}_p^p\right \rvert }{\max( \norm{q_t - x_1}_p,\norm{q_t - x_2}_p  )^{p-1} } \leq \frac{\left \lvert  \norm{q_t - x_1}_p^p - \norm{q_t - x_2}_p^p\right \rvert }{ \norm{0.5(x_1 - x_2)}_p^{p-1} } \,.
\]
Alternatively, we may also use the trivial bound 
\[
  \abs{\pnorm{q_t - x_1}{p} - \pnorm{q_t - x_2}{p}}
    \le \pnorm{x_1 - x_2}{p} \,.
\]

Let $i_0$ be the largest positive integer such that 
\[
\frac{i_0}{2^{i_0}} \geq  0.1 d^{-1} \norm{x_1 - x_2}_p \,.
\]
Note that $i_0 \geq 2$.

We can now bound the total loss of our algorithm, say $L$, as follows:
\begin{align*}
  L & \leq  \sum_{i = 1}^{i_0} \norm{x_1 - x_2}_p O\left(  p^3d^32^i (i + \log dp) \right) \\
    & \quad + \sum_{i = i_0 + 1}^{\infty}  \frac{\left \lvert  \norm{q_t - x_1}_p^p - \norm{q_t - x_2}_p^p\right \rvert}{\norm{0.5(x_1 - x_2)}_p^{p-1}} \cdot  O\left(  p^3d^32^i (i + \log dp)  \right) \\
    & \leq \sum_{i = 1}^{i_0} \norm{x_1 - x_2}_p O\left(  p^3d^32^i (i + \log dp)  \right) \\
    & \quad + \frac{1}{{\norm{0.5(x_1 - x_2)}_p^{p-1}}}  \sum_{i = i_0 + 1}^{\infty}  O( d^2p^2 (10 \cdot 2^id^2)^{-(p-1)}) \cdot  O\left(  p^3d^32^i (i + \log dp)  \right)  \\
    & \leq  \poly(d,p)  \frac{1}{\norm{x_1 - x_2}_p}  \sum_{i=0}^{\infty} \frac{ 1 + i}{2^{(p-2)i}} \\
    & \leq \poly(d,p)  \cdot \frac{1}{\norm{x_1 - x_2}_p} \cdot  \frac{1}{(1 - 2^{-(p-2)})^2} \\
    & \leq \frac{\poly(d,p)  }{\norm{x_1 - x_2}_p} \cdot \left(\frac{1}{p-2}\right)^2\,.
\end{align*}

\end{proof}

\begin{proof}[Proof of Theorem \ref{thm:general-p-norm-k-point}]
In light of Theorem \ref{thm:two-to-many}, it suffices to argue that {\sc Multiscale Nearest Neighbor Learning} is a potential-based algorithm.  Indeed, the corresponding potential is defined as follows. Let $i_0$ be the largest positive integer such that 
\[
\frac{i_0}{2^{i_0}} \geq  0.1 d^{-1} \Delta \,.
\]
Define 
\begin{align*}
  P_t &= \sum_{i=1}^{i_0} \Delta \log \frac{\Vol\left( S_i^{(t)}\right)}{\Vol(B_{2p'(2D_i + 1)}(0, 0.1 \cdot (p \delta_i)^p) )}  \\
    & \quad + \sum_{i = i_0 + 1}^{\infty} \frac{d^2p^2}{(10 \cdot 2^id^2)^{p-1}(0.5\Delta)^{p-1}}\log \frac{\Vol\left( S_i^{(t)}\right)}{\Vol(B_{2p'(2D_i + 1)}(0, 0.1 \cdot (p \delta_i)^p) )} \,.
\end{align*}
The proof of Theorem \ref{thm:general-p-norm} immediately implies that $P_t$ is a valid potential and we are done.
\end{proof}

%% file: apx-general-regions.tex
\section{Learning General Convex Regions: Missing Proofs}
\label{apx:general-regions}

In this appendix, we present the full proof of Theorem~\ref{thm:general-regions}.

Our construction is based on only picking points on the surface of the unit ball (i.e. the unit hypersphere). There are two key factors to ensuring that the algorithm accrues enough total error; we need to ensure that (i) each time we choose a point, it could lie in either of the two regions and (ii) the point is sufficiently far from the region it is not in. To guarantee the former, we choose our points by considering a separating hyperplane between the two regions so far.

To guarantee the latter, we will choose our points to be $\epsilon$-far from each other, for some $\epsilon$ based on the total number of points we need to choose $T$. This, combined with the fact that we chose points on the surface of the unit ball, implies that the minimum penalty for a mistake is $\Omega(\epsilon^2)$. Note that for both of our guarantees to simultaneously work out, we actually need our separating hyperplane to go through the origin (so the resulting intersection has enough surface area for this part of the argument).

The requirement that points be $\epsilon$ far from each other limits the total number of points we can choose. Roughly speaking, each point removes on the order of a $\epsilon^{d-2}$-fraction of the (hyper-)surface area of the intersection of the separating hyperplane with the unit hypersphere. Maximizing $\epsilon$ while ensuring we can pick $T$ points yields the desired bound in the theorem statement.

Our construction will utilize two technical results concerning the geometry of high-dimensional objects. One is a result of Klee regarding the existence of separating hyperplanes for convex cones \citep{klee1955separation}. It uses the following notation.

\begin{definition}[\hspace{1sp}\cite{klee1955separation}]
\label{def:cones}
  A $0$-cone is a closed convex cone having the origin (denoted $0$) as its vertex. For a $0$-cone $A$, $A'$ denotes the linear subspace $A \cap -A$.
\end{definition}

The technical lemma gives conditions for a strict linear separator between two such convex cones.

\begin{theorem}[\hspace{1sp}\cite{klee1955separation} Theorem 2.7]
\label{thm:separating-hyperplane}
  Suppose $E$ is a separable normed linear space, $A$ and $B$ are $0$-cones in $E$, $A$ is locally compact, and $A \cap B = \{0\}$. Then $E$ admits a continuous linear functional $\hyperplane$ such that $\hyperplane < 0$ on $A \setminus A'$, $\hyperplane = 0$ on $A' \cup B'$, and $\hyperplane > 0$ on $B \setminus B'$.
\end{theorem}

The next technical lemma upper bounds the surface area of a hypersphere cap relative to the entire hypersphere, implying that many such caps are needed to cover the hypersphere.

\begin{definition}
  We denote the $d$-dimensional unit hypersphere as $\hypersphere{d} \triangleq \{x \in \mathbb{R}^d \mid \lVert x \rVert_2 = 1\}$ (note that this is the surface of the unit ball). We denote its (hyper-)surface area as $\hyperspherearea{d}$.
  
  Next, we denote the cap centered at $v \in \hypersphere{d}$ of angle $\phi \in [0, \pi/2]$ as $\hypercap{d}{v}{\phi} \triangleq \hypersphere{d} \cap \{x \in \mathbb{R}^d \mid \inner{x}{v} \le \cos \phi\}$. We denote its (hyper-)surface area as $\hypercaparea{d}{\phi}$.
\end{definition}

\begin{lemma}
\label{lem:hypercap}
  For all integer $d \ge 1$ and $\phi \in [0, \pi / 2]$,
  \begin{align*}
    \frac{\hypercaparea{d}{\phi}}{\hyperspherearea{d}} \le \phi^{d-1}
  \end{align*}
\end{lemma}

\begin{proof}
  We will prove this bound by building on the analysis of \cite{li2011concise}. The surface area of a $d$-dimensional hypersphere is well known to be
  \begin{align*}
    \hyperspherearea{d} &= \frac{2 \pi^{d/2}}{\Gamma (d/2)}\text{,}
  \end{align*}
  where the gamma function $\Gamma$ represents the standard extension of the factorial function.
  
  We begin with the following observation of Li:
  \begin{align*}
    \hypercaparea{d}{\phi} &= \frac{2 \pi^{(d-1)/2}}{\Gamma((d-1)/2)}
      \int_0^\phi \sin^{d-2} x \,dx\text{.}
  \end{align*}
  
  Next, we apply the fact that $\sin x \le x$ for $x \ge 0$.
  \begin{align*}
    \hypercaparea{d}{\phi}
      &\le \frac{2 \pi^{(d-1)/2}}{\Gamma((d-1)/2)}
      \int_0^\phi x^{d-2} \,dx \\
      &= \frac{2 \pi^{(d-1)/2}}{\Gamma((d-1)/2)} 
      \left[ \frac{1}{d-1} \phi^{d-1} \right] \\
      &= \frac{\pi^{(d-1)/2}}{\Gamma((d+1)/2)} \phi^{d-1}
    \frac{\hypercaparea{d}{\phi}}{\hyperspherearea{d}} \\
      &\le \frac{\Gamma(d/2)}{2 \sqrt{\pi} \Gamma((d+1)/2)} \phi^{d-1}
  \end{align*}
  
  We are almost done; we just need to show that $\frac{\Gamma(d/2)}{2 \sqrt{\pi} \Gamma((d+1)/2)} \le 1$. After $\Gamma(3/2)$, subsequent half-values of $\Gamma$ are increasing and so the desired claim is trivially true for $d \ge 3$. We manually check $d = 1$ and $d = 2$, noting that $\Gamma(1/2) = \sqrt{\pi}$, $\Gamma(1) = 1$, and $\Gamma(3/2) = \sqrt{\pi}/2$.
  \begin{align*}
    \frac{\Gamma(1/2)}{2 \sqrt{\pi} \Gamma(2/2)}
      &= \frac12 \\
    \frac{\Gamma(2/2)}{2 \sqrt{\pi} \Gamma(3/2)}
      &= \frac1\pi
  \end{align*}
  This completes the proof.
\end{proof}

We are now ready to give the full proof for Theorem~\ref{thm:general-regions}, which is restated below for convenience.

\generalregionstheorem*

\begin{proof}
  Our counterexample depends on two parameters: the dimension $d$ and an error parameter $\epsilon > 0$. We will choose $\epsilon = 1 / T^{1/(d-2)}$, where $T$ is the total number of time steps.
  
  Our construction restricts itself to choosing points on $\hypersphere{d}$. The construction begins with the opposite points $x_1 = e_1 \triangleq (1, 0, 0, \ldots, 0)$ and $x_2 = -e_1$. $x_1$ has true label $1$ and $x_2$ has true label $2$. We will use $A$ to denote the conic hull of the points with true label $1$ and $B$ to denote the conic hull of the points with true label $2$. By construction, $A$ and $B$ are $0$-cones. We will maintain the two invariants that (i) $A \cap B = \{0\}$ and (ii) $A' = B' = \{0\}$, which we will use when invoking Theorem~\ref{thm:separating-hyperplane}.
  
  We now explain how to generate subsequent points $x_t$ for $t \ge 3$. By Theorem~\ref{thm:separating-hyperplane}, we know there is a separating hyperplane $\hyperplane$ that passes through the origin and strictly separates $A \setminus \{0\}$ from $B \setminus \{0\}$. We pick an arbitrary hyperplane $\hyperplane$ that satisfies the previous statement and examine its intersection with the hypersphere, $X \triangleq \{x \mid \hyperplane(x) = 0\} \cap \hypersphere{d}$.
  
  Since we $\hyperplane$ passes through the origin, $X$ is a $(d-1)$-dimensional unit hypersphere. We want to choose our next point $x_t$ to be a point in $X$ that is at least $\epsilon$-far from all previously chosen points $x_1, x_2, ..., x_{t-1}$. This condition rules out $(t-1)$ hypersphere caps of angles at most $2 \arcsin \frac{\epsilon}{2}$. Observe that as $T$ increases and $\epsilon$ decreases, this angle bound scales as $O(\epsilon)$. By Lemma~\ref{lem:hypercap}, we have (hyper-)surface area remaining as long as $(t-1) \cdot O(\epsilon^{d-2}) \le 1$, i.e. we can pick up to $(1 / \epsilon)^{d-2} = T$ total points safely. We will pick an arbitrary point $x_t$ satisfying our $\epsilon$-far rule.
  
  We assign this point $x_t$ a uniform random true label between $1$ and $2$. It remains to establish that our two invariants are still satisfied. For the sake of contradiction, assume that (i) is no longer true and that there is a non-origin point $z$ in the intersection $A \cap B$. Without loss of generality, $x_t$ was assigned true label 1 and $f$ used to be strictly positive (negative respectively) on $A \setminus \{0\}$ ($B \setminus \{0\}$ respectively). We can deduce the following.
  \begin{align*}
    \hyperplane(z) &< 0 \\
    \hyperplane(z) &= \hyperplane\left(\sum_{x \text{ has true label 1}} \alpha_x x\right) \\
         &= \sum_{x \text{ has true label 1}} \alpha_x \hyperplane(x) \\
         &\ge 0
  \end{align*}
  for some values $\alpha_x \ge 0$. The first inequality follows from the fact that $z$ is in $B \setminus \{0\}$ and the second inequality follows from the fact that all points with true label 1 either are in $A \setminus \{0\}$ before this round or are $x_t$. We've reached a contradiction and conclude that invariant (i) remains true.
  
  We establish invariant (ii) much in the same way. Again, for the sake of contradiction we assume that (ii) is false. Without loss of generality, assume that $x_t$ was assigned true label 1 and that now there exists a non-origin point $z$ in $A' = A \cap -A$. We follow a similar line of deduction.
  \begin{align*}
    \hyperplane(z) &= \hyperplane\left(\sum_{x \text{ has true label 1}} \alpha_x x\right) \\
    \hyperplane(z) &\ge 0 \\
    \hyperplane(z) &= \hyperplane\left(\sum_{x \text{ has true label 1}} \beta_x x\right) \\
    \hyperplane(z) &\le 0 \\
    \hyperplane(z) &= 0
  \end{align*}
  for some values $\alpha_x \ge 0$ and $\beta_x \le 0$. But this implies that $\alpha_x = \beta_x = 0$ forall $x$ that have true label $1$ and are not the most recent point $x_t$. We are left with $z = \alpha_{x_t} x_t = \beta_{x_t} x_t$ for some $\alpha_{x_t} \ge 0$ and $\beta_{x_t} \le 0$, but this can only hold if $\alpha_{x_t} = \beta{x_t} = 0$ (recall that $x_t$ is not the origin by construction). But then $z$ is the origin, which contradicts our assumption. Hence our assumption is false, i.e. invariant (ii) is indeed maintained.
  
  All that remains is the add up error of our algorithm. Since we chose the true label for the final $T-2$ points uniformly at random, any online algorithm will make $\frac{T-2}{2}$ mistakes in expectation.
  
  \input{fig/tikz-circle}
  
  How much error does the algorithm accrue for every mistake? Recall that our construction kept all points at least $\epsilon$ away from each other and all points were on the hypersphere (the surface of the unit ball). Figure~\ref{fig:circle} illustrates the situation with respect to one point. We can solve for the distance that our point of interest $x$ is away from the unit ball minus its cap, which is an upper bound for the convex hull of all other points.
  \begin{align*}
    h(2-h) &= w^2 \\
    2h &= w^2 + h^2 \\
    2h &= \epsilon^2 \\
    h &= \epsilon^2 / 2
  \end{align*}
  
  Hence the algorithm accrues $\Omega(\epsilon^2)$ expected error in each timestep, and so overall any online algorithm is expected to accrue $\Omega\left(T\epsilon^2\right)$ total error. By our choice of $\epsilon$, this equals the desired bound.
\end{proof}

%% file: fig/tikz-circle.tex
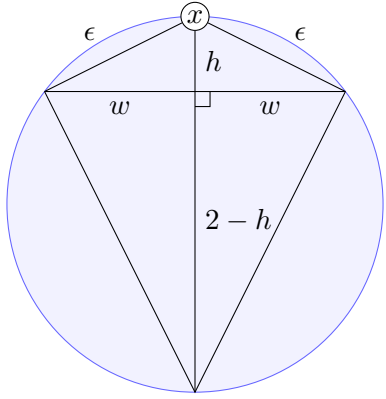
\begin{figure}
\centering
\begin{tikzpicture}[%
  auto,
  scale=1.0,
  hollownode/.style={
    circle,
    draw=black,
    fill=white,
    inner sep=1pt,
  },
  ]
  \filldraw[color=blue!60, fill=blue!5] (0, 0) circle (2.5);
  \node[hollownode] (p1) at (0, 2.5) {$x$};
  \coordinate (p2) at (2, 1.5) {};
  \coordinate (p3) at (0, -2.5) {};
  \coordinate (p4) at (-2, 1.5) {};
  \coordinate (intersection) at (0, 1.5) {};
  
  \draw (p1) -- node[above right=1mm] {$\epsilon$} (p2);
  \draw (p2) -- (p3) -- (p4);
  \draw (p4) -- node [above left=1mm] {$\epsilon$} (p1);
  
  \draw (p1) -- node[right] {$h$} (intersection);
  \draw (intersection) -- node[above right] {$2-h$} (p3);
  \draw (p2) -- node[below] {$w$} (intersection);
  \draw (intersection) -- node[below] {$w$} (p4);
  
  \draw (0, 1.3) -- (0.2, 1.3) -- (0.2, 1.5);
\end{tikzpicture}
\caption{Circular cross-section of the $d$-dimensional unit hypersphere, $\hypersphere{d}$. The two points $\epsilon$ away from $x$ (in Euclidean distance) and the diametrically opposite point are marked, forming a kite. The diagonals of this kite result in a right triangle of interest to our analysis, with height $h$, width $w$, and hypotenuse $\epsilon$.}
\label{fig:circle}
\end{figure}

%% file: apx-regret-vs-distance.tex
\section{Distance-based losses for nearest neighbor partitions}
\label{app:regret-vs-distance}

Semantically, there is a difference between our margin-independent loss for convex sets and our margin-independent loss for nearest neighbors: in the first setting, the loss of labelling a query $q$ with a label $i$ is the \textit{distance} we must move $q$ so that it would have label $i$, whereas in the second setting our loss is our \textit{regret} under the assumption that we receive utility $-\delta(q, x_i)$ from labelling query $q$ with label $i$ (``regret'' meaning the difference between our utility and the optimal utility for this query). 

The goal of this appendix is to provide some additional justification for why we study the second, regret-based notion of loss for nearest neighbors instead of the distance-based notion of loss (which has the benefit that it extends to arbitrary partitions, whereas the regret-based notion of loss requires some notion of similarity metric). Roughly, the main reason for doing this is that unlike the regret-based definition, the distance-based definition of loss is not continuous in the centers $x_i$: 

\begin{lemma}
Given $k$ points $x_1, x_2, \dots, x_k \in \R^d$, let $f_i(q; x_1, x_2, \dots, x_k) = \min_{q' \in R_i} ||q - q'||_2$, where $R_i = \{q \in \R^d \mid ||q - x_i|| \leq ||q - x_j|| \;\forall j \in [k]\}$ (i.e., $f_i(q; x)$ is the minimum distance needed to move $q$ so that it is closer to $x_i$ than to any of the other points $x_j$).  Then $f_{i}$ is \textit{not} continuous in the variables $x_1, x_2, \dots, x_n$.
\end{lemma}
\begin{proof}
We will show this for $k=3$ and $d=1$. Fix $x_1 = 0$ and $x_2 = 1$, and consider the function $f_3$. Note that when $x_3 = \eps > 0$ (for some $\eps < 0.1$), $f_3(1/2; x_1, x_2, x_3) = 0$; in particular, $1/2$ is closer to $x_3$ than to any other center. But when $x_3 = -\eps < 0$, then $f_3(1/2; x_1, x_2, x_3) = (1+\eps)/2 > 1/2$; in other words, we now need to move $q=1/2$ by at least $1/2$ to bring $q$ closer to $x_3$ than any other center. This shows that $f_3$ is discontinuous in $x_3$.
\end{proof}

This fact makes the distance-based notion more technically cumbersome to work with, both from an analysis perspective and a learning perspective; for example, this fact means it is not simply sufficient for a learning algorithm to approximate the values $x_i$ to high precision, since even infinitesimal changes in the values of $x_i$ can drastically change the partition into regions $R_i$. Nonetheless, it is not clearly impossible  to construct a learning algorithm with good distance-based loss guarantees for the nearest neighbor partition setting (in particular, the convex sets constructed in the proof of \Cref{thm:general-regions} cannot be implemented as a nearest-neighbor partition with a small number of cnenters), and this is an interesting open question.